\documentclass{article}

   \PassOptionsToPackage{numbers, compress}{natbib}

\usepackage[preprint]{neurips_2022}

\usepackage[utf8]{inputenc} 
\usepackage[T1]{fontenc}    
\usepackage{hyperref}       
\usepackage{url}            
\usepackage{booktabs}       
\usepackage{amsfonts}       
\usepackage{nicefrac}       
\usepackage{microtype}      
\usepackage{xcolor}         

\title{Federated Learning on Riemannian Manifolds}

\author{%
  Jiaxiang Li\\
  Department of Mathematics\\
  University of California, Davis\\
  Davis, CA 95616 \\
  \texttt{jxjli@ucdavis.edu} \\
    \And
  Shiqian Ma\thanks{\url{https://www.math.ucdavis.edu/\~sqma}} \\
  Department of Mathematics\\
  University of California, Davis\\
  Davis, CA 95616 \\
  \texttt{sqma@ucdavis.edu} \\
}

\begin{document}

\maketitle

\begin{abstract}
Federated learning (FL) has found many important applications in smart-phone-APP based machine learning applications. Although many algorithms have been studied for FL, to the best of our knowledge, algorithms for FL with nonconvex constraints have not been studied. This paper studies FL over Riemannian manifolds, which finds important applications such as federated PCA and federated kPCA. We propose a Riemannian federated SVRG (\texttt{RFedSVRG}) method to solve federated optimization over Riemannian manifolds. We analyze its convergence rate under different scenarios. Numerical experiments are conducted to compare \texttt{RFedSVRG} with the Riemannian counterparts of \texttt{FedAvg} and \texttt{FedProx}. We observed from the numerical experiments that the advantages of \texttt{RFedSVRG} are significant. 
\end{abstract}

\section{Introduction}\label{section_intro}
Federated learning (FL) has drawn lots of attentions recently due to its wide applications in modern machine learning. Canonical FL aims at solving the following finite-sum problem~\cite{konevcny2016federated,mcmahan2017communication,kairouz2021advances}:
\begin{equation}\label{finite_sum}
\min_{x\in\mathbb{R}^d} f(x):=\frac{1}{n}\sum_{i=1}^{n}f_i(x),
\end{equation}
where each of the $f_i$ (or the data associated with $f_i$) is stored in different client/agent that could have different physical locations and different hardware. This makes the mutual connection impossible~\cite{konevcny2016federated}. Therefore, there is a central server that can collect the information from different agents and output a consensus that minimizes the summation of the loss functions from all the clients. The aim of such a framework is to utilize the computation resources of different agents while still maintain the data privacy by not sharing data among all the local agents. Thus the communication is always between the central server and local servers. This setting is commonly observed in modern smart-phone-APP based machine learning applications~\cite{konevcny2016federated}. We emphasize that we always consider the heterogeneous data scenario where the functions $f_i$'s might be different and have different optimal solutions. This problem is inherently hard to solve because each local minima will empirically diverge the update from the global optimum~\cite{li2020federated,mitra2021linear}.

In this paper, we consider the following FL problem over a Riemannian manifold $\M$:
\begin{equation}\label{problem_finite_sum}
\min_{x\in\M} f(x):=\frac{1}{n}\sum_{i=1}^{n}f_i(x)
\end{equation}
where $f_i:\M\rightarrow\RR$ are smooth but not necessarily (geodesically) convex. It is noted that most FL algorithms are designed for the unconstrained setting and convex constraint setting \cite{konevcny2016federated,mcmahan2017communication, karimireddy2020scaffold, li2019convergence, malinovskiy2020local, charles2021convergence, pathak2020fedsplit, mitra2021linear}, and FL problems with nonconvex constraints such as \eqref{problem_finite_sum} have not been considered. The main difficulty for solving \eqref{problem_finite_sum} lies in aggregating points over a nonconvex set, which may lead to the situation where the averaging point is outside of the constraint set.

One motivating application of \eqref{problem_finite_sum} is the federated kPCA problem 
\begin{equation}\label{problem_kPCA}
\min_{X\in\St(d, r)} f(X):=\frac{1}{n}\sum_{i=1}^{n}f_i(X),\ \mbox{ where } f_i(X)=-\frac{1}{2}\tr(X^\top A_i X),
\end{equation}
where $\St(d, r)=\{X\in\RR^{d\times r}| X^\top X=I_r\}$ denotes the Stiefel manifold, and $A_i$ is the covariance matrix of the data stored in the $i$-th local agent. When $r=1$, \eqref{problem_kPCA} reduces to classical PCA
\begin{equation}\label{problem_PCA}
\min_{\|x\|_2=1} f(x):=\frac{1}{n}\sum_{i=1}^{n}f_i(x),\ \mbox{ where } f_i(x)=-\frac{1}{2}x^\top A_i x.
\end{equation}
Existing FL algorithms are not applicable to \eqref{problem_kPCA} and \eqref{problem_PCA} due to the difficulty on aggregating points on nonconvex set. 




\subsection{Main Contributions}
We focus on designing efficient federated algorithms for solving \eqref{problem_finite_sum}. Our main contributions are: 

\begin{enumerate}[leftmargin=*]
\item We propose a Riemannian federated SVRG algorithm (\texttt{RFedSVRG}) for solving \eqref{problem_finite_sum}. We prove that the convergence rate of our RFedSVRG algorithm is $\mathcal{O}(1/\epsilon^2)$ for obtaining an $\epsilon$-stationary point. This result matches that of its Euclidean counterparts~\cite{mitra2021linear}. To the best of our knowledge, this is the first algorithm for solving FL problems over Riemannian manifolds with convergence guarantees. 

    \item The main novelty of our \texttt{RFedSVRG} algorithm is a consensus step on the tangent space of the manifold. We compare this new approach with the widely used Karcher mean approach. We show that our method achieves certain "regularization" property and performs very well in practice. 
    \item We conduct extensive numerical experiments on our method for solving the PCA \eqref{problem_PCA} and kPCA \eqref{problem_kPCA} problems with both synthetic and real data. The numerical results demonstrate that our \texttt{RFedSVRG} algorithm significantly outperforms the Riemannian counterparts of two widely used FL algorithms: \texttt{FedAvg} \cite{mcmahan2017communication} and \texttt{FedProx} \cite{li2020federated}.
\end{enumerate}


\subsection{Related Work}

\textbf{Federated optimization.} 
The most natural idea for FL is the \texttt{FedAvg} algorithm \cite{mcmahan2017communication}, which averages local gradient descent updates and yields a good empirical convergence. However in the data heterogeneous situation, \texttt{FedAvg} suffers from the client-drift effect that each local client will drift the solution towards the minimum of their own local loss function \cite{karimireddy2020scaffold, li2019convergence, malinovskiy2020local, charles2021convergence, pathak2020fedsplit, mitra2021linear}. Many ideas were studied to resolve this issue. For example, \cite{li2020federated} proposed the \texttt{FedProx} algorithm, which regularizes each of the local gradient descent update to ensure that the local iterates are not far from the previous consensus point. The \texttt{FedSplit}~\cite{pathak2020fedsplit} was proposed later to further mitigate the client-drift effect and convergence results were obtained for convex problems. 
\texttt{FedNova}~\cite{wang2020tackling} was also proposed to improve the performance of \texttt{FedAvg}, however it still suffers from a fundamental speed-accuracy conflict under objective heterogeneity~\cite{mitra2021linear}. Variance reduction techniques were also incorporated to FL leading to two new algorithms: federated SVRG (\texttt{FSVRG}) \cite{konevcny2016federated} and \texttt{FedLin} \cite{mitra2021linear}. These two algorithms require transmitting the full gradient from the central server to each local client for local gradient updates, therefore require more communication between clients and the central server. Nevertheless, \texttt{FedLin} achieves the theoretical lower bound for strongly convex objective functions~\cite{mitra2021linear} with an acceptable amount of increase in the communication cost.

\textbf{Decentralized optimization on manifolds.} Decentralized distributed optimization on manifold has also drawn attentions in recent years~\cite{chen2021decentralized, shah2017distributed,alimisis2021distributed}. Under this setting, each local agent solves a local problem and then the central server takes the consensus step. The consensus step is usually done by calculating the Karcher mean on the manifold~\cite{tron2012riemannian,shah2017distributed}, or calculating the minimizer of the sum of the square of the Euclidean distances in the embedded submanifold case~\cite{chen2021decentralized}. Such consensus steps usually require solving an additional problem inexactly with no exact convergence rate guarantee~\cite{tron2012riemannian, chen2021local}. 

It is worth mentioning that the PCA problem under federated learning setting has been considered in the literature \cite{grammenos2020federated}. The proposed method in~\cite{grammenos2020federated} relies on the SVD of data matrices and a subspace merging technique, which is very different from our method. The aim of the algorithm in~\cite{grammenos2020federated} is to achieve $(\epsilon,\delta)$-differential privacy. In contrast, we mainly consider the convergence rate of our method. Therefore our work is totally different from \cite{grammenos2020federated}. 

\section{Preliminaries on Riemannian Optimization}

In this part, we briefly review the basic tools we use for optimization on Riemannian manifolds~\cite{absil2009optimization,lee2006riemannian,Tu2011manifolds,boumal2022intromanifolds}. Due to the limit of space, more detailed discussions are given in supplementary material \ref{appendix_manifold}. Suppose $\M$ is an $m$-dimensional Riemannian manifold with Riemannian metric $g:T\M\times T\M\rightarrow\RR$. We first review the notion of the Riemannian gradients.
\begin{definition}[Riemannian gradients]
    For a Riemannian manifold with Riemannian metric $g$, the Riemannian gradient for $f\in C^\infty(\M)$ is the unique tangent vector $\grad f(x)\in T_x\M$ such that $df(\xi) = g(\grad f, \xi),\ \forall \xi\in T_x\M$, where $d f$ is the differential of function $f$ defined as $d f(\xi):=\xi(f)$.
\end{definition}

For the convergence analysis, we also need the notion of exponential mapping and parallel transport. We first review the definition of exponential mapping
\begin{definition}[Exponential mapping]
    Given $x\in\M$ and $\xi\in T_x\M$, the exponential mapping $\Exp_x$ is defined as a mapping from $T_x\M$ to $\M$ s.t. $\Exp_x(\xi):= \gamma(1)$ with $\gamma$ being the geodesic with $\gamma(0)=x$, $\Dot{\gamma}(0)=\xi$. A natural corollary is $\Exp_x(t\xi):= \gamma(t)$ for $t\in[0, 1]$. Another useful fact is $d(x,\Exp_x(\xi))=\|\xi\|_x$ since $\gamma'(0)=\xi$ which preserves the speed.
\end{definition}
Throughout this paper, we always assume that $\M$ is complete, so that $\Exp_x$ is always defined for every $\xi\in T_x\M$. For $\forall x,y\in\M$, the inverse of the exponential mapping $\Exp_{x}^{-1}(y)\in T_x\M$ is called the logarithm mapping, and we have $d(x,y)=\|\Exp_{x}^{-1}(y)\|_x$, which will be a useful fact in the convergence analysis. We now present the definition of parallel transport.
\begin{definition}[Parallel transport]
    Given a Riemannian manifold $(\M, g)$ and two points $x,y\in\M$, the parallel transport $P_{x\rightarrow y}:T_x\M\rightarrow T_y\M$\footnote{Notice that the existence of parallel transport depends on the curve connecting $x$ and $y$, which is not a problem for complete Riemannian manifold because we always take the unique geodesic that connects $x$ and $y$.} is a linear operator which keeps the inner product: $\forall \xi,\zeta\in T_x\M$, we have $\langle P_{x\rightarrow y}\xi, P_{x\rightarrow y}\zeta\rangle_y = \langle\xi, \zeta\rangle_x$.
\end{definition}
Parallel transport is useful since the Lipschitz condition for the Riemannian gradient requires moving the gradients in different tangent spaces "parallel" to the same tangent space.

We now present the definition of Lipschitz smoothness and convexity on Riemannian manifolds, which will be utilized in our convergence analysis.
\begin{definition}[$L$-smoothness on manifolds]\label{assumption_manifold_smooth}
$f$ is called Lipschitz smooth on manifold $\M$ if there exists $L\geq0$ such that the following inequality holds for function $f$:
\begin{equation}\label{eq:lgsmoothness1}
    \|\grad f(y) - P_{y\rightarrow x}\grad f(x)\|\leq L d(x,y).
\end{equation}
For complete Riemannian manifold, we have~\cite{zhang2016first}:
\begin{equation}\label{eq:lgsmoothness2}
    f(y) \leq f(x)+\left\langle g_{x}, \Exp_{x}^{-1}(y)\right\rangle_{x}+\frac{L_{g}}{2} d^{2}(x, y),\ \forall x,y\in\M.
\end{equation}
\end{definition}

The definition of geodesic convexity is given below (see, e.g., \cite{zhang2016first}).
\begin{definition}[Geodesic convex]\label{assumption_geodesic_convex}
    A function $f\in C^1(\M)$ is geodesically convex if for all $x,y\in\M$, there exists a geodesic $\gamma$ such that $\gamma(0)=x$, $\gamma(1)=y$ and 
    $$
    f(\gamma(t))\leq (1-t)f(x)+t f(y),\ \forall t\in[0,1].
    $$
    Or equivalently,
    $$
    f(y)\geq f(x) + \langle \grad f(x), \Exp_{x}^{-1}(y) \rangle_x.
    $$
\end{definition}


\section{The RFedSVRG Algorithm}
The most challenging task for FL on Riemannian manifolds is the consensus step. Suppose the central server receives $x^{(i)}$, $i\in S_t\subset[n]$ from each of the local clients at round $t$, the question is how the central server aggregates the points to output a unique consensus. In Euclidean space, the most straightforward way is to take the average $\frac{1}{k}\sum_{i\in S_t}x^{(i)}$ with $k=|S_t|$. However, this approach does not apply to the Riemannian setting due to the loss of linearity: the arithmetic average of points can be outside of the manifold. A natural choice for the consensus step on the manifold is to take the Karcher mean of the points \cite{tron2012riemannian}:
\begin{equation}\label{karcher_mean}
    x_{t+1}\leftarrow\argmin_x \frac{1}{k}\sum_{i\in S_t}d^2(x, x^{(i)}),
\end{equation}
where $x_{t+1}$ is the next iterate point on the central server. This is a natural generalization of the arithmetic average because $d^2(x,y)=\|x-y\|^2$ in Euclidean space. However, solving \eqref{karcher_mean} can be time consuming in practice. 

We propose the following tangent space consensus step: 
\begin{equation}\label{tangent_space_mean}
    x_{t+1}\leftarrow \Exp_{x_{t}}\left(\frac{1}{k}\sum_{i\in S_t}\Exp_{x_{t}}^{-1}(x^{(i)})\right),
\end{equation}
where we project each of the point $x_t^{(i)}$ back to the tangent space $T_{x_t}\M$ and then take their average on the tangent space. The consensus step \eqref{tangent_space_mean} has several advantages over the Karcher mean method \eqref{karcher_mean}. First, \eqref{tangent_space_mean} is of closed-form and easy to compute. Second, \eqref{tangent_space_mean} still coincides with the arithmetic mean when the manifold reduces to the Euclidean space. Third, the tangent space mean \eqref{tangent_space_mean} can easily be extended to the following moving average mean: 
\[
\Exp_{x_{t}}\left(\frac{\beta}{k}\sum_{i\in S_t}\Exp_{x_{t}}^{-1}(x^{(i)})\right),
\] 
which corresponds to $(1-\beta)x_t+\frac{\beta}{k}\sum_{i\in S_t}x^{(i)}$ in the Euclidean space, while the Karcher mean cannot be easily extended in this scenario. Last, \eqref{tangent_space_mean} has the following "regularization" property as the distance between two consensus points can be controlled, and the Karcher mean method \eqref{karcher_mean} does not have this kind of property.

\begin{lemma}\label{lemma_regularization_tangent_mean}
    For the update defined in \eqref{tangent_space_mean}, it holds that
    $$
        d(x_{t+1}, x_t)\leq \frac{1}{k}\sum_{i\in S_t} d(x^{(i)}, x_t).
    $$
\end{lemma}

To further illustrate this "regularization" property of the tangent space mean \eqref{tangent_space_mean}, we consider an (extreme) example on the unit sphere $\mathcal{S}^2$ (see Figure \ref{fig:consensus}) . Here we take $x_t$ on the north pole and two point from the local server as $x^{(1)}$ and $x^{(2)}$, also $\xi^{(i)}=\Exp_{x_t}^{-1}(x^{(i)})\in T_{x_t}\M$. Then the tangent space mean \eqref{tangent_space_mean} would yield the original point $x_t$, whereas the Karcher mean could yield any point on the vertical great circle, depending on the starting point in solving the optimization problem \eqref{karcher_mean}.
\begin{figure}
    \centering
    
    \begin{tikzpicture}[
      point/.style = {draw, circle, fill=black, inner sep=0.7pt}, scale = 0.8
    ]
    \def\rad{2cm}
    \coordinate (O) at (0,0); 
    \coordinate (N) at (0,\rad); 
    
    \filldraw[ball color=white] (O) circle [radius=\rad]; 
    \draw[dashed] 
      (0, \rad) arc [start angle=90,end angle=-90,x radius=5mm,y radius=\rad];
    \draw 
      (0, \rad) arc [start angle=90,end angle=270,x radius=5mm,y radius=\rad];
      
    \begin{scope}[xslant=0.5,yshift=\rad,xshift=-2] 
    \filldraw[fill=gray!10,opacity=0.3]
      (-4.5,1) -- (2.5,1) -- (3,-1) -- (-4,-1) -- cycle;
    \node at (2,0.6) {$T_{x_t}\mathcal{S}^2$};  
    \end{scope}
    
    \draw[dashed] 
      (N) node[above] {$x_t$} -- (O) node[below] {$O$};
    \node[point] at (N) {};
    
    \draw[line width=1pt,blue,-stealth](0,\rad)--(pi,\rad) node[anchor=north east]{$\xi^{(1)}$};
    \draw[line width=1pt,red,-stealth](0,\rad)--(-pi,\rad) node[anchor=south west]{$\xi^{(2)}$};
    
    \node[point] at (2, 0) {};
    \node[right] at (2, 0) {$x^{(1)}$};
    
    \node[point] at (-2, 0) {};
    \node[left] at (-2, 0) {$x^{(2)}$};
    
    \end{tikzpicture}
    
    \caption{Comparison of two consensus methods on $\mathcal{S}^2$}
    \label{fig:consensus}
\end{figure}
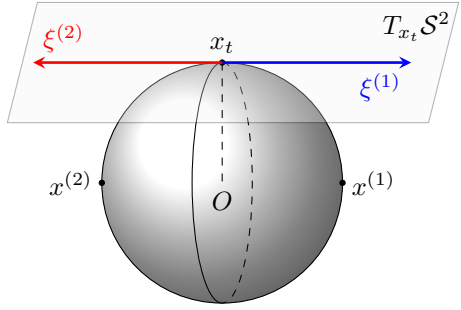


Our \texttt{RFedSVRG} algorithm is presented in Algorithm \ref{manifold_fedsvrg}, which is a non-trivial manifold extension of the FSVRG algorithm \cite{konevcny2016federated}. 
For \texttt{RFedSVRG}, the local gradient update becomes
\begin{equation}\label{local_update_fedsvrg}
    x_{\ell+1}^{(i)}\leftarrow \Exp_{x_{\ell}^{(i)}}\left[-\eta^{(i)} \left(\grad f_i(x_{\ell}^{(i)}) - P_{x_t\rightarrow x_{\ell}^{(i)}}(\grad f_i(x_t) - \grad f(x_t))\right)\right],
\end{equation}
which matches the existing manifold SVRG work \cite{zhang2016fast}. The introduction of the parallel transport $P_{x_t\rightarrow x_{\ell}^{(i)}}$ is necessary because we need to "transport" all the vectors to the same tangent space to conduct addition and subtraction. The algorithm utilizes the gradient information at the previous iterate $\grad f(x_t)$, thus avoids the "client-drift" effect and correctly converges to the global stationary points. This is confirmed by both the theory and the numerical experiments.

\begin{algorithm}[ht]

\SetKwInOut{Input}{input}
\SetKwInOut{Output}{output}
\SetAlgoLined
\Input{$n$, $k$, $T$, $\{\eta^{(i)}\}$, $\{\tau_i\}$}
\Output{\textbf{Option 1:} $\Tilde{x}=x_T$; or \textbf{Option 2:} $\Tilde{x}$ is uniformly sampled from $\{x_1,...,x_T\}$}
    \For{$t=0,...,T-1$}{
        Uniformly sample $S_t\subset [n]$ with $|S_t|=k$\;
        \For{each agent $i$ in $S_t$}{
            Receive $x_0^{(i)}=x_t$ from the central server\;
            \For{$\ell=0,...,\tau_i-1$}{
                Take the local gradient step \eqref{local_update_fedsvrg}.
            }
            Send $\hat{x}^{(i)}$ (obtained by one of the following options) to the central server 
            \begin{itemize}
                \item {\textbf{Option 1:}  $\hat{x}^{(i)}=x_{\tau_i}^{(i)}$;}
                \item {\textbf{Option 2:}} $\hat{x}^{(i)}$ is uniformly sampled from $\{x_{1}^{(i)},...,x_{\tau_i}^{(i)}\}$\;
            \end{itemize}
        }
        The central server aggregates the points by the tangent space mean \eqref{tangent_space_mean}\;
    }
    \caption{Riemannian FedSVRG Algorithm (RFedSVRG)}\label{manifold_fedsvrg}
\end{algorithm}

\section{Convergence analysis}
In this section we analyze the convergence behaviour of the \texttt{RFedSVRG} algorithm (Algorithm \ref{manifold_fedsvrg}). Before we proceed to the convergence results, we briefly review the necessary assumptions, which are standard assumptions for optimization on manifolds~\cite{zhang2016first,boumal2018global}.
\begin{assumption}[Smoothness]\label{assumption_smoothness}
    Suppose $f_i$ is $L_i$-smooth as defined in  \eqref{assumption_manifold_smooth}. It implies that $f$ is $L$-smooth with $L=\sum_{i=1}^{n}L_i$.
\end{assumption}

Now we give the convergence rate results for Algorithm \ref{manifold_fedsvrg}. Specifically, Theorem \ref{thm_nonconvex1} gives the convergence rate of Algorithm \ref{manifold_fedsvrg} with $\tau_i=1$, Theorem \ref{thm_nonconvex1.1} gives the convergence rate of Algorithm \ref{manifold_fedsvrg} with $\tau_i>1$, and Theorem \ref{thm_geodesic_convex} gives the convergence rate of Algorithm \ref{manifold_fedsvrg} when the objective function is geodescially convex. 
\begin{theorem}[Nonconvex, Algorithm \ref{manifold_fedsvrg} with $\tau_i=1$]\label{thm_nonconvex1}
    Suppose the problem \eqref{problem_finite_sum} satisfies Assumption \ref{assumption_smoothness}. If we run Algorithm \ref{manifold_fedsvrg} with \textbf{Option 1} in Line 8, $\eta^{(i)}\leq \frac{1}{L}$ and $\tau_i=1$ (i.e. only one step of gradient update for each agent), then the \textbf{Option 1} of the output of Algorithm \ref{manifold_fedsvrg} satisfies:
    \begin{equation}\label{thm-ineq}
        \min_{t=0,...,T}\|\grad f(x_t)\|^2\leq \mathcal{O}\left(\frac{L (f(x_0)-f(x^*))}{T}\right).
    \end{equation}
\end{theorem}

\begin{remark}\label{remark_multiple_innersteps}
    Our proof of Theorem \ref{thm_nonconvex1} relies heavily on the choice of $\tau_i=1$ and the consensus step \eqref{tangent_space_mean}. 
    When $\tau_i>1$, we need to introduce multiple exponential mappings at multiple points for each iteration, which makes the convergence analysis much more challenging due to the loss of linearity. Moreover, the aggregation step makes the situation even worse. However, we are able to show the convergence of Algorithm \ref{manifold_fedsvrg} with $\tau_i>1$ when $k=1$. Our numerical experiments show the effectiveness of the \texttt{RFedSVRG} algorithm with both $\tau_i=1$ and $\tau_i>1$.
\end{remark}

To prove the convergence of Algorithm \ref{manifold_fedsvrg} with $\tau_i> 1$, we also need the following regularization assumption over the manifold $\M$~\cite{zhang2016fast}.
\begin{assumption}[Regularization over manifold]\label{assumption_regu_manifold}
    The manifold is complete and there exists a compact set $\mathcal{D}\subset \M$ (diameter bounded by $D$) so that all the iterates of Algorithm \ref{manifold_fedsvrg} and the optimal points are contained in $\mathcal{D}$. The sectional curvature is bounded in $[\kappa_{\min}, \kappa_{\max}]$. Moreover, we denote the following key geometrical constant that captures the impact of manifold:
    \begin{equation}\label{zeta_eq}
        \zeta= \begin{cases}\frac{\sqrt{\left|\kappa_{\min }\right|} D}{\tanh \left(\sqrt{\left|\kappa_{\min }\right|} D\right)}, & \text { if } \kappa_{\min }<0 \\ 1, & \text { if } \kappa_{\min } \geq 0.\end{cases}
    \end{equation}
\end{assumption}
Notice that this assumption holds when the manifold is a sphere or a Stiefel manifold (since they are compact). Now we are ready to give the convergence rate result of Algorithm \ref{manifold_fedsvrg} with $\tau_i>1$ and $k=1$, the proof of which is inspired by~\cite{zhang2016fast}.

\begin{theorem}[Nonconvex, Algorithm \ref{manifold_fedsvrg} with $\tau_i>1$ and $k=1$]\label{thm_nonconvex1.1}
    Suppose the problem \eqref{problem_finite_sum} satisfies Assumptions \ref{assumption_smoothness} and \ref{assumption_regu_manifold}. If we run Algorithm \ref{manifold_fedsvrg} with \textbf{Option 2} in Line 8, $k=1$, $\tau_i=\tau>1$, $\eta^{(i)}=\eta\leq \mathcal{O}(\frac{1}{n L \zeta^2})$, then the \textbf{Option 2} of the output of Algorithm \ref{manifold_fedsvrg} satisfies:
    \[
        \E\|\grad f(\Tilde{x})\|^2\leq \mathcal{O}\left(\frac{\rho (f(x_0)-f(x^*))}{\tau T}\right),
    \]
    where $\rho$ is an absolute constant specified in the proof and the expectation is taken with respect to the random index $i$, as well as the randomness introduced by the \textbf{Option 2}.
\end{theorem}

Finally, we have the convergence result when the objective function of \eqref{problem_finite_sum} is geodesically convex.

\begin{theorem}[Geodesic convex]\label{thm_geodesic_convex}
    Suppose the problem \eqref{problem_finite_sum} satisfies Assumption \ref{assumption_smoothness} and \ref{assumption_regu_manifold}. Also the functions $f_i$'s are geodesically convex (see Definition \ref{assumption_geodesic_convex}) in $\mathcal{D}$ (as in Assumption \ref{assumption_regu_manifold}). If we run Algorithm \ref{manifold_fedsvrg} with \textbf{Option 1} in Line 8, $\tau_i=1$, $S_t=[n]$ (full parallel gradient), and  $\eta=\eta^{(1)}=\cdots=\eta^{(n)}\leq \frac{1}{2 L}$, then the \textbf{Option 1} of the output of Algorithm \ref{manifold_fedsvrg} satisfies:
    \begin{equation}
         f(x_T) - f^*\leq \mathcal{O}\left(\frac{L d^2(x_0,x^*)}{T}\right).
    \end{equation}
\end{theorem}

\section{Numerical experiments}
We now show the performance of RFedSVRG and compare it with two natural ideas for solving \eqref{finite_sum}: Riemannian FedAvg (\texttt{RFedAvg}) and Riemannian FedProx (\texttt{RFedProx}), which are natural extensions of FedAvg \cite{mcmahan2017communication} and FedProx \cite{li2020federated} to the Riemannian setting. Algorithms \texttt{RFedAvg} and \texttt{RFedProx} are descried in Algorithm \ref{manifold_fedavg} and Algorithm \ref{manifold_fedprox} in the supplementary material. We conducted our experiments on a desktop with Intel Core 9600K CPU, 32GB RAM and NVIDIA GeForce RTX 2070 GPU. For the codes of operations on Riemannian manifolds we used the ones from the \texttt{Manopt} and \texttt{PyManopt} packages~\cite{manopt,pymanopt}. Since the logarithm mapping (the inverse of the exponential mapping) on the Stiefel manifold is not easy to compute \cite{zimmermann2021computing}, we adopted the projection-like retraction~\cite{absil2012projection} and the inverse of it~\cite{kaneko2012empirical} to approximate the exponential and the logarithm mappings, respectively. 

We tested the three algorithms on PCA \eqref{problem_PCA} and kPCA \eqref{problem_kPCA} problems. For both problems, we measure the norm of the global Riemannian gradients. Additionally, we also measure the sum of principal angles \cite{knyazev2012principal} for kPCA. \footnote{For the loss $f$ in \eqref{problem_kPCA}, note that $f(X)=f(XQ)$ for any orthogonal matrix $Q\in\RR^{r\times r}$. As a result, the optimal solution of $f(X)$ only represents the eigen-space corresponds to the $r$-largest eigenvalues. Therefore we need the principal angles to measure the angles between the subspaces.}



\subsection{Comparison of the two consensus methods \eqref{karcher_mean} and \eqref{tangent_space_mean}}

We first compare the two consensus methods \eqref{karcher_mean} and \eqref{tangent_space_mean}. To this end, we randomly generate $x_t$ and $k=100$ points $x^{(i)}$ on the unit ball $\mathcal{S}^{d-1}$ with different dimensions $d$. We then compare the distances $\frac{1}{k}\sum_i d^2(x_{t}, x^{(i)})$, $\frac{1}{k}\sum_i d^2(x_{t+1}, x^{(i)})$ and $d^2(x_t, x_{t+1})$, as well as the CPU time for computing them. Note that the smaller these distances are, the better. To calculate the Karcher mean, we run the Riemannian gradient descent method starting at $x_t$ until the norm of the Riemannian gradient is smaller than $\epsilon=10^{-6}$. The results are shown in Table \ref{table:consensus}. 
From Table \ref{table:consensus} we see that the tangent space mean \eqref{tangent_space_mean} is indeed better than Karcher mean \eqref{karcher_mean} in terms of both quality and CPU time. 

\begin{table}[t]
\begin{center}
\caption{Comparison of the two consensus methods \eqref{karcher_mean} and \eqref{tangent_space_mean}. Here $h(x):=\frac{1}{k}\sum_i d^2(x^{(i)},x)$, CPU time is in seconds and the experiments are repeated and averaged over 10 times.}\label{table:consensus}
\begin{small}
\begin{tabular}{c|c|c|c|c|c|c|c}
\hline
\multirow{2}{*}{Dim $d$} & \multirow{2}{*}{$h(x_t)$} & \multicolumn{3}{c}{Karcher mean \eqref{karcher_mean}} & \multicolumn{3}{|c}{Tangent space mean \eqref{tangent_space_mean}} \\
\cline{3-8}
& & $d^2(x_{t+1}, x_t)$ & $h(x_{t+1})$ & Time & $d^2(x_{t+1}, x_t)$ & $h(x_{t+1})$ & Time \\
\hline
100 & 2.478 & 2.469 & 2.813 & 0.706 & 0.025 & 2.427 & 0.004 \\
\hline
200 & 2.472 & 2.484 & 2.804 & 0.641 & 0.025 & 2.422 & 0.004 \\
\hline
500 & 2.469 & 2.469 & 2.795 & 0.725 & 0.024 & 2.421 & 0.005 \\
\hline
\end{tabular}
\end{small}
\end{center}
\end{table}

\subsection{Experiments for PCA and kPCA on synthetic data}

In this section, we report the results of the three algorithms for solving PCA \eqref{problem_PCA} and kPCA \eqref{problem_kPCA} on synthetic data. We first generate the data $X_i\in\RR^{d\times p}$ whose entries are drawn from standard normal distribution. We then set $A_i:=X_i X_i^\top$. Notice that under this experiment setting the data in different agents are homogeneous in distribution, which provides a mild environment for comparing the behavior of the proposed algorithms. We test highly heterogeneous real data later.

\paragraph{Experiments on PCA.} We first test the three algorithms on the standard PCA problem \eqref{problem_PCA}. We test our codes with different numbers of agents $n$ and set $k=n/10$ as the number of clients we pick up for each round. We terminate the algorithms if the number of rounds of communication exceeds 600. We sample $10000$ data points in $\RR^{100}$ and partition them into $n$ agents, each of which contains equal number of data. We test \texttt{RFedSVRG} with one iteration for each local agents, i.e. $\tau_i=1$ and test \texttt{RFedAvg} and \texttt{RFedProx} with $\tau_i=5$ iterations in \eqref{temp6}. We use the constant stepsizes for all three algorithms, and take $\mu=n/10$ for each choice of $n$. The results are presented in Figure \ref{fig:pca_changing_nk_norm}, from which we see that only \texttt{RFedSVRG} can efficiently decrease $\|\grad f(x_t)\|$ to an acceptable level. 

\begin{figure}[t!]
    \begin{center}
    \setcounter{subfigure}{0}
    \subfigure[$n=500$]{\includegraphics[width=0.32\textwidth]{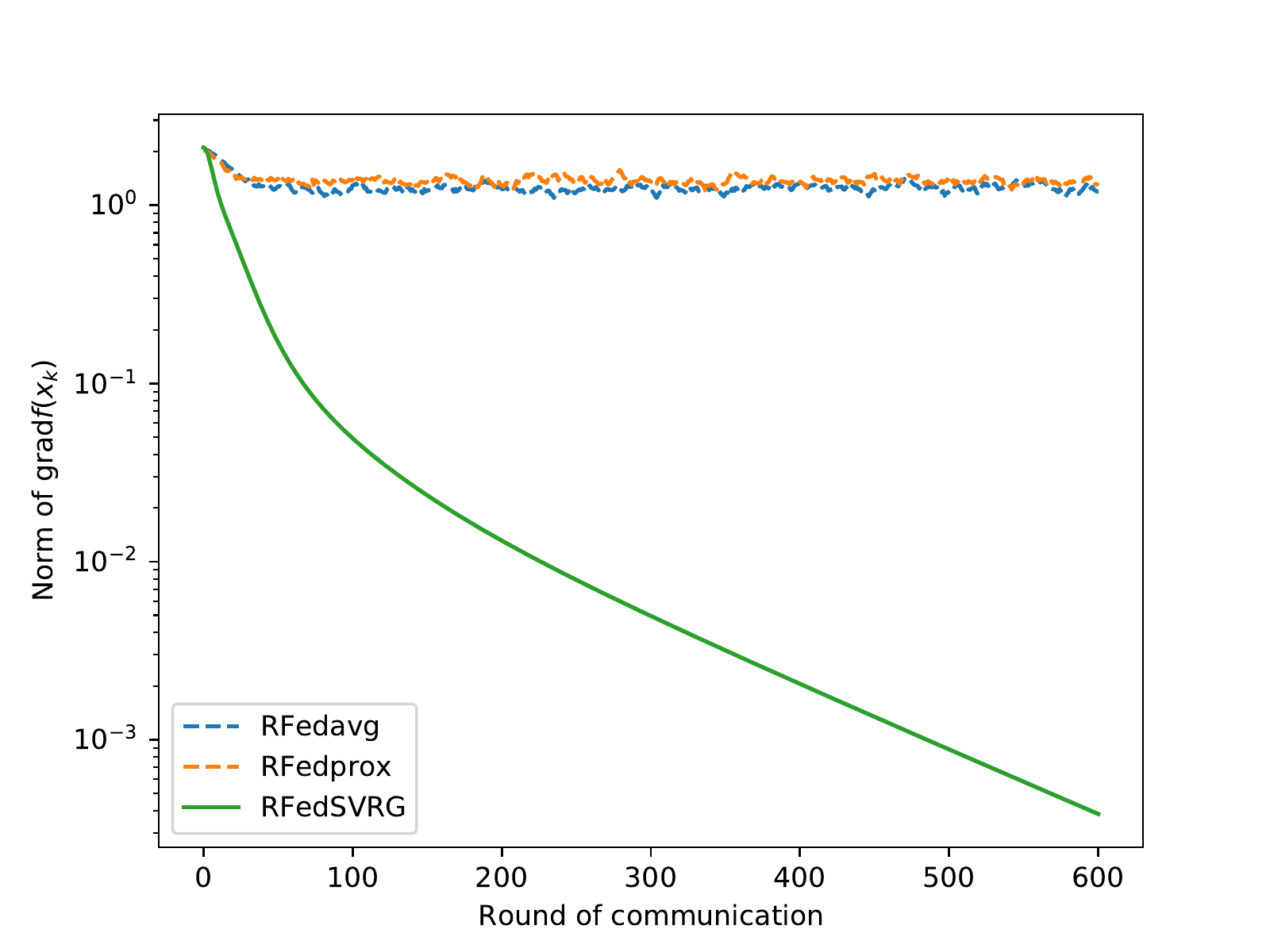}}
    \subfigure[$n=1000$]{\includegraphics[width=0.32\textwidth]{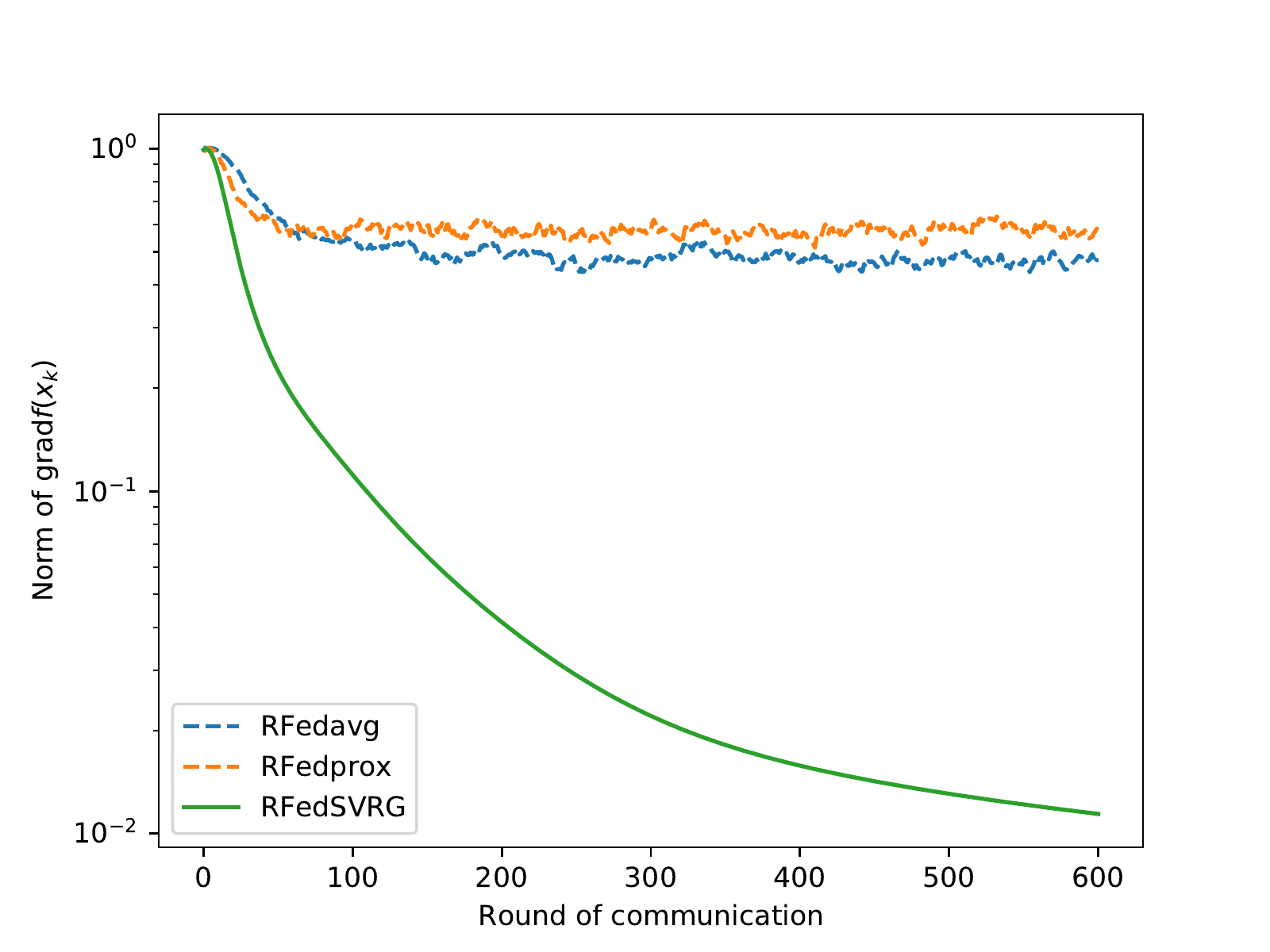}}
    \subfigure[$n=2500$]{\includegraphics[width=0.32\textwidth]{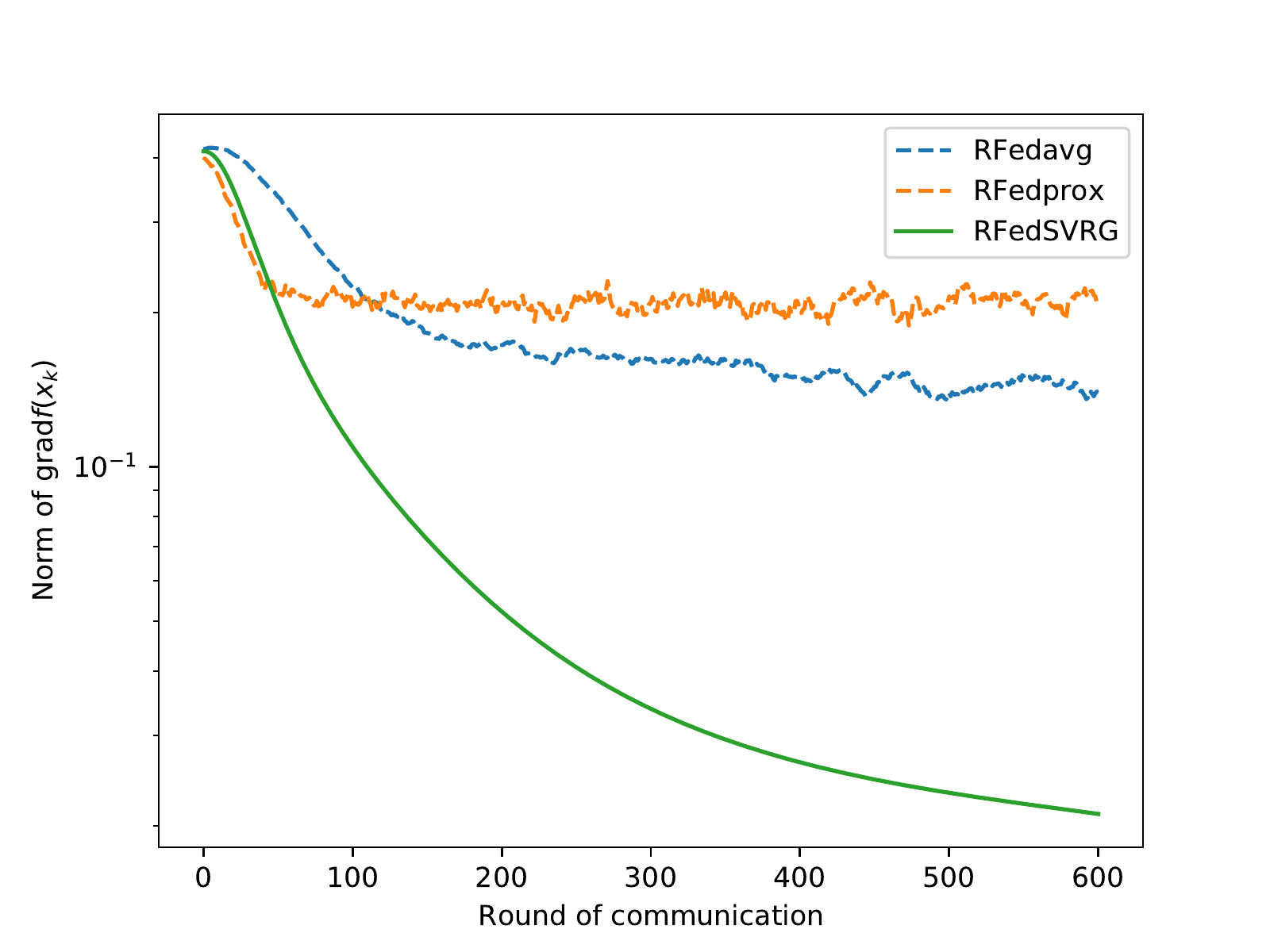}}
    \caption{Results for PCA \eqref{problem_PCA}. The y-axis denotes $\|\grad f(x_t)\|$. For each figure, the experiments are repeated and averaged over 10 times.}
    \label{fig:pca_changing_nk_norm}
    \end{center}
\end{figure}

\paragraph{Experiments on kPCA.} We now test the three algorithms on the kPCA problem \eqref{problem_kPCA}. In the first experiment we sample $10000$ data points in $\RR^{200}$ and partition them into $n$ agents, each of which contains equal number of data. We test our codes with different number of agents $n$, and again set $k=n/10$. Here we take $(d, r)=(200, 5)$. The results are given in Figure \ref{fig:kpca_changing_nk}, where we see that \texttt{RFedSVRG} can efficiently decrease $\|\grad f(x_t)\|$ and the principal angle in all tested cases.

In the second experiment we test the effect of the number of inner loops $\tau_i$.
We generate $10000$ standard Gaussian vectors. We set $(d,r) = (200,5)$, $k=10$ and $n=100$ so that $p=100$. We choose $\tau=[1, 10, 50, 100]$ for the inner steps for all three algorithms. The results are presented in Figure \ref{fig:kpca_changing_tau}. From this figure we again observe the great performance of \texttt{RFedSVRG}.

\begin{figure}[t!]
    \begin{center}
    \subfigure{\includegraphics[width=0.23\textwidth]{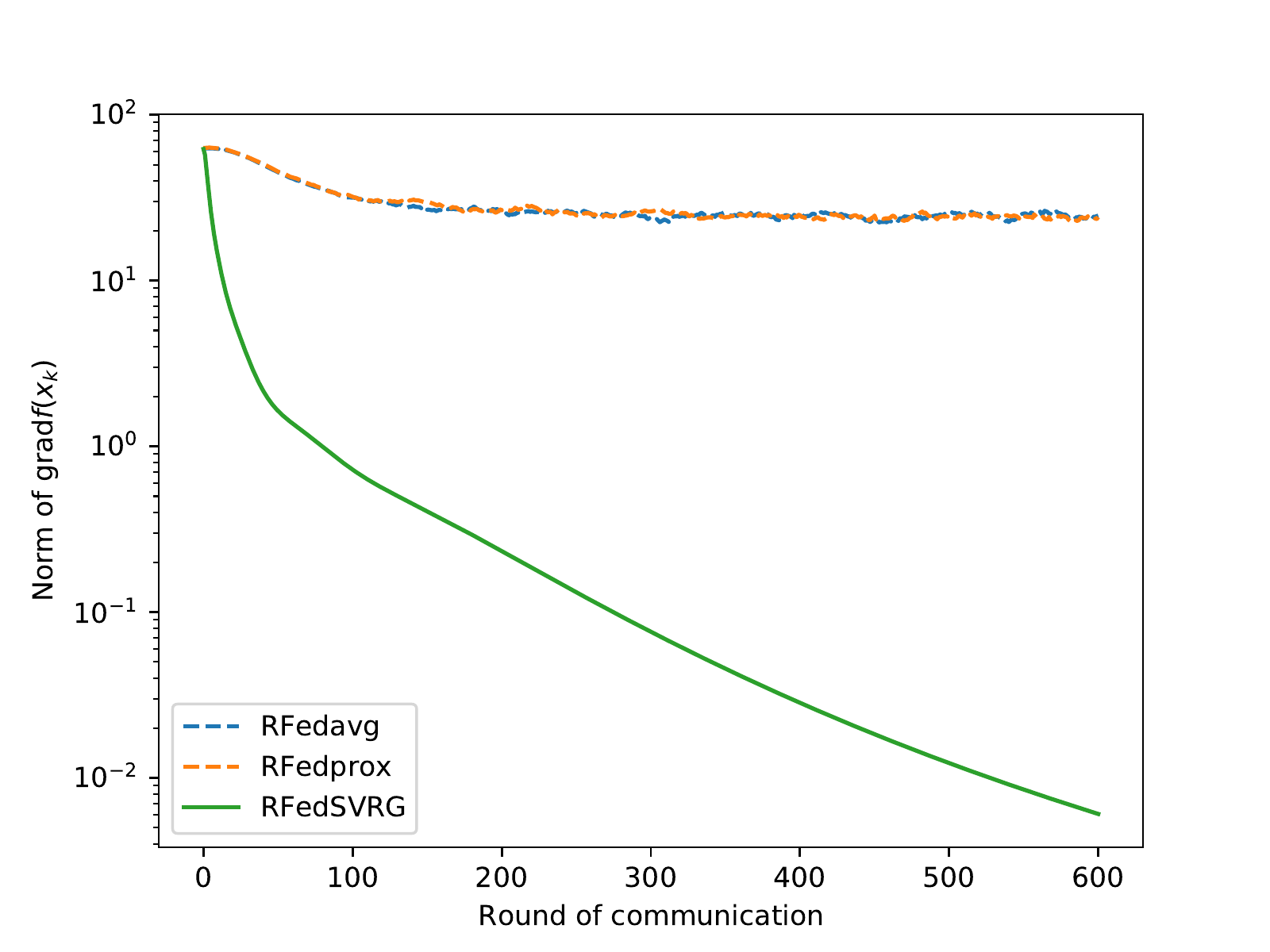}}
    \subfigure{\includegraphics[width=0.23\textwidth]{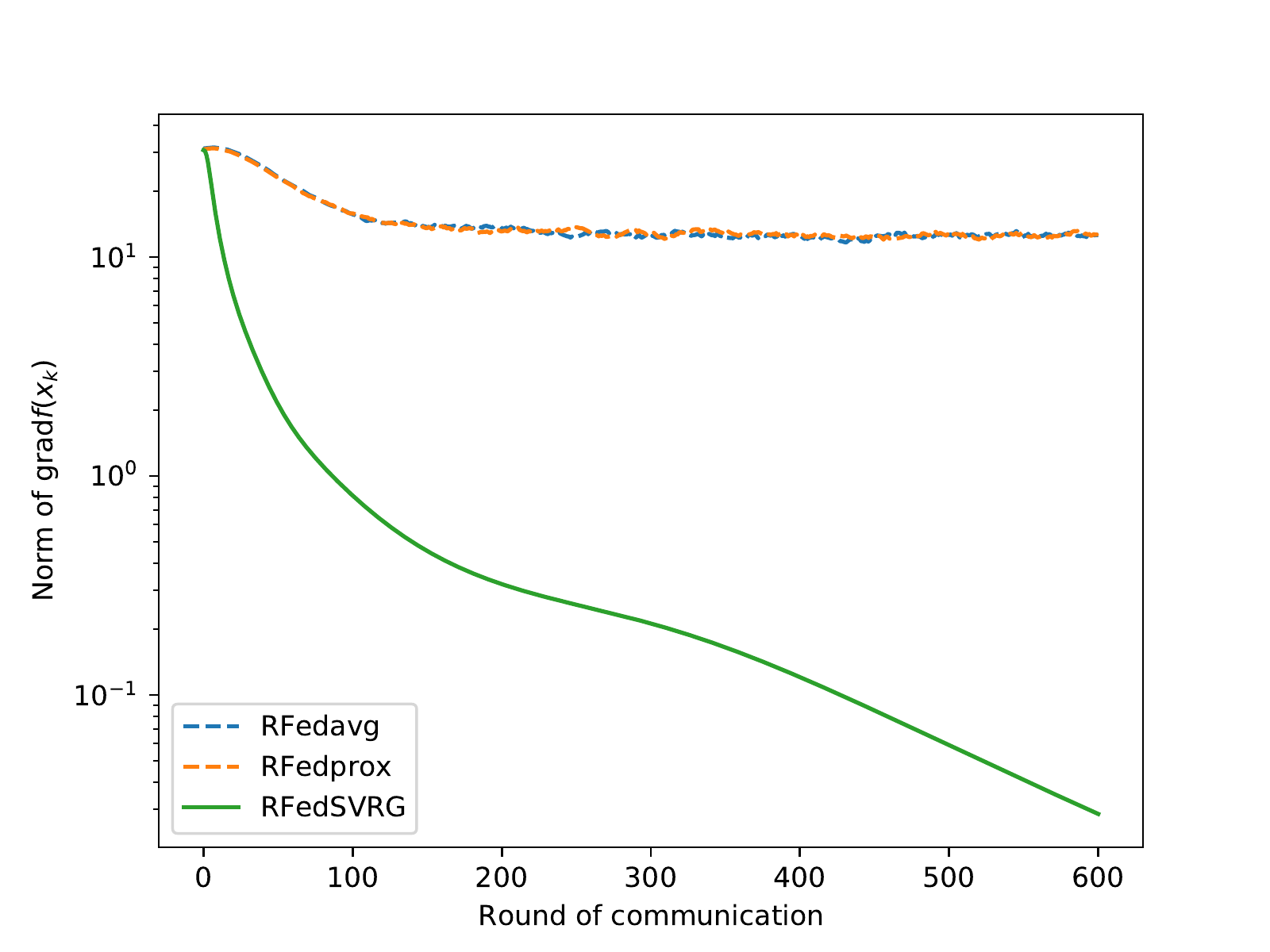}}
    \subfigure{\includegraphics[width=0.23\textwidth]{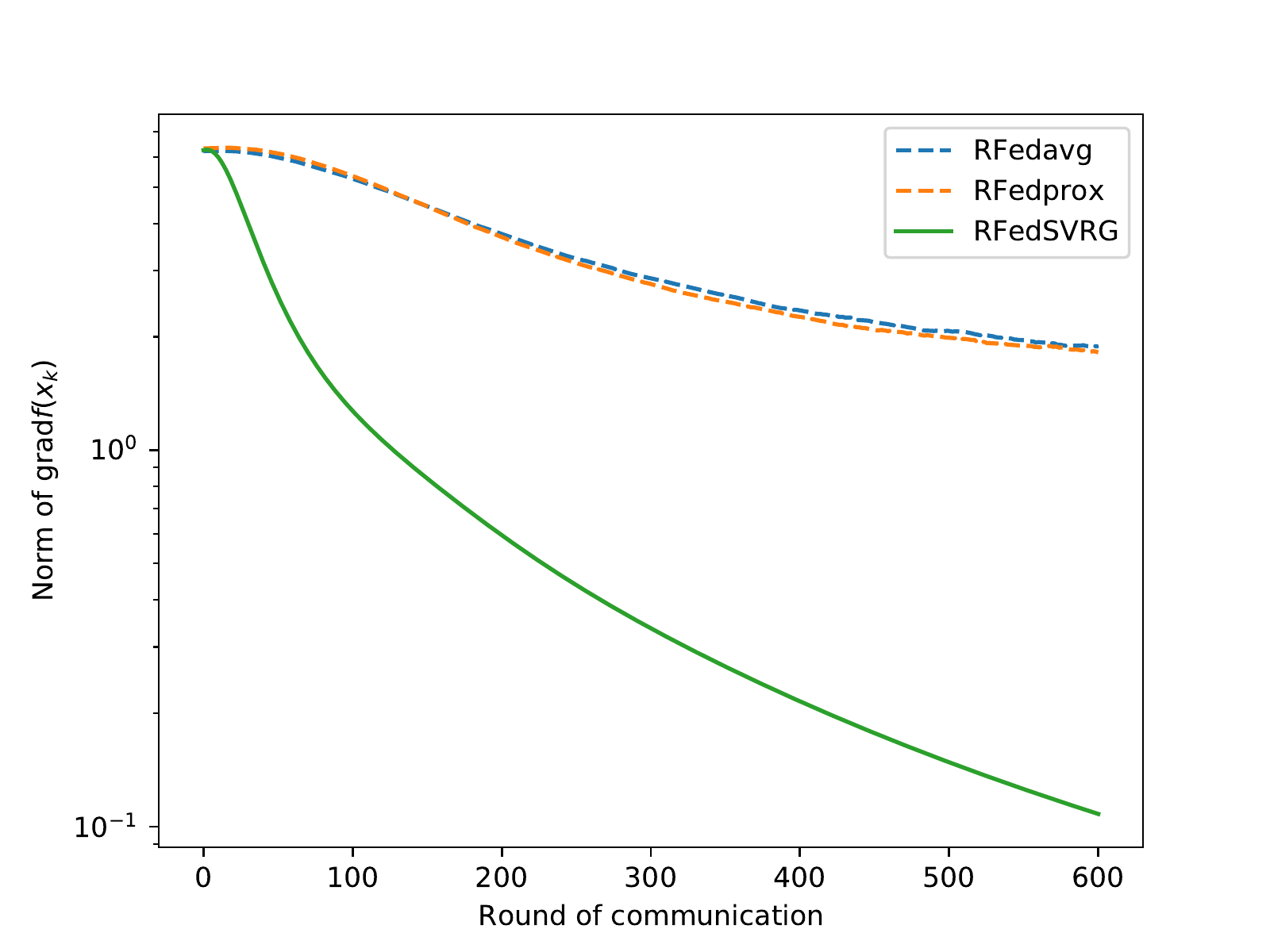}}
    \subfigure{\includegraphics[width=0.23\textwidth]{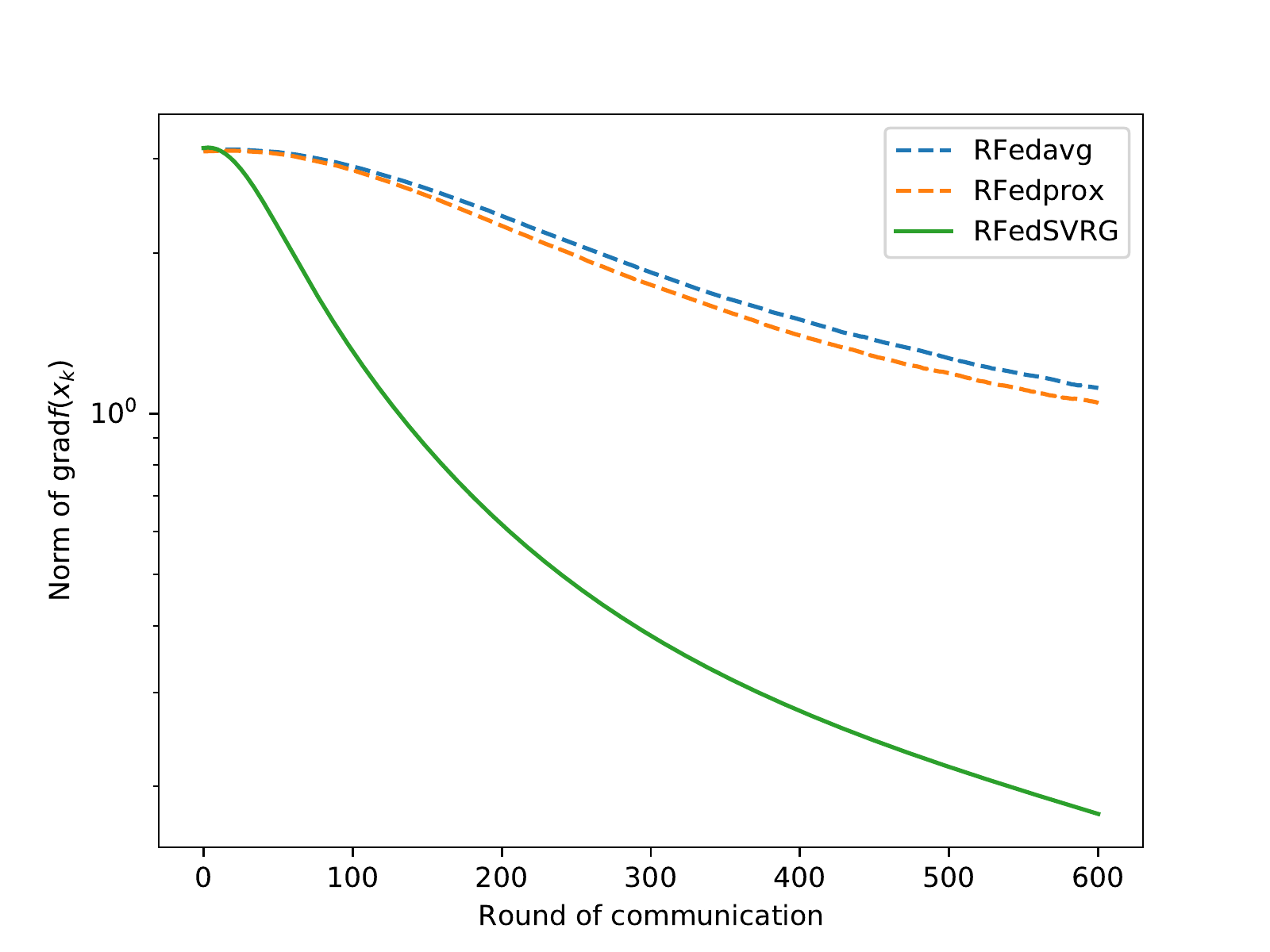}}
    
    \setcounter{subfigure}{0}
    \subfigure[$(n, k)=(50, 5)$]{\includegraphics[width=0.23\textwidth]{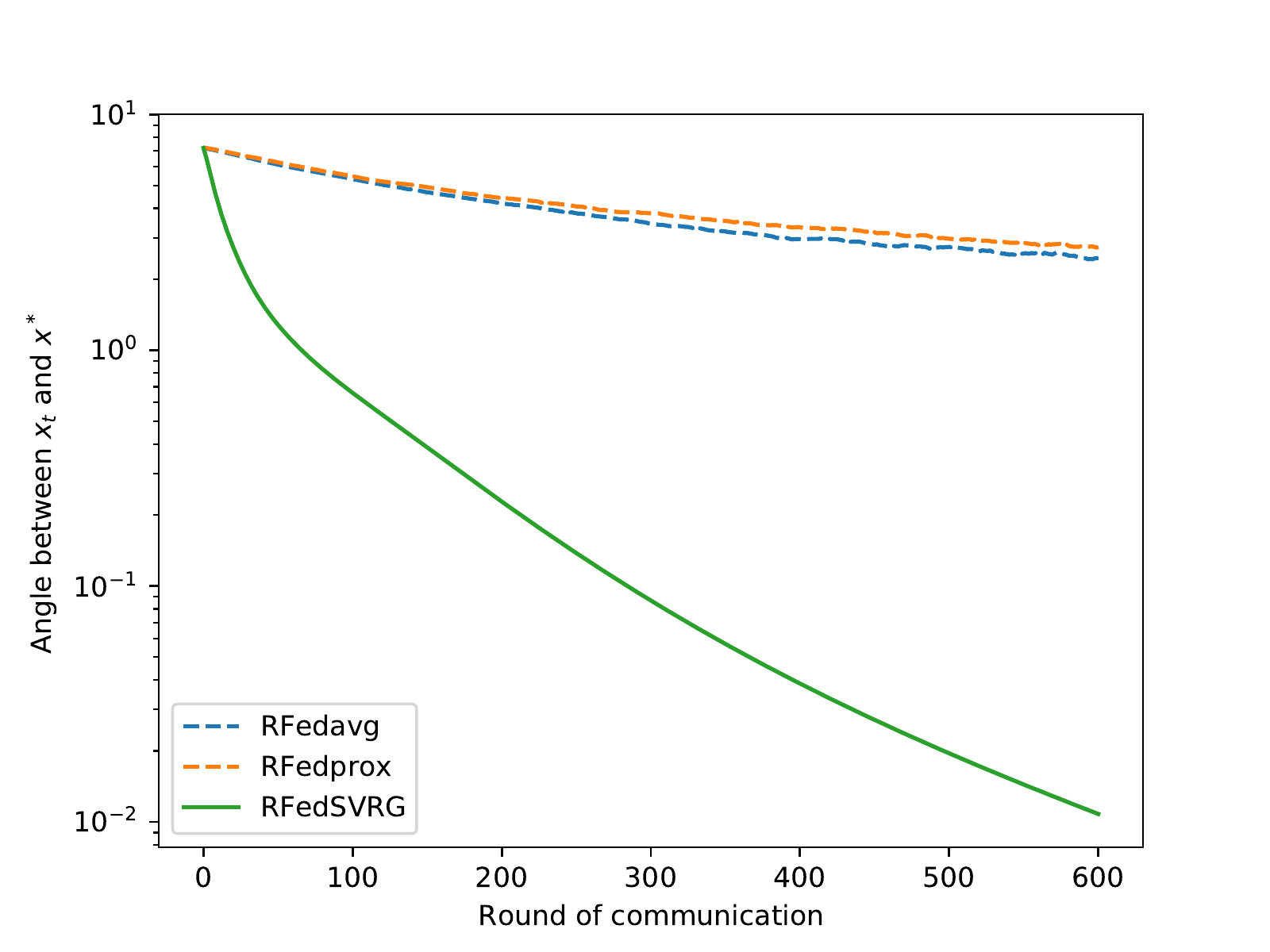}}
    \subfigure[$(n, k)=(100, 10)$]{\includegraphics[width=0.23\textwidth]{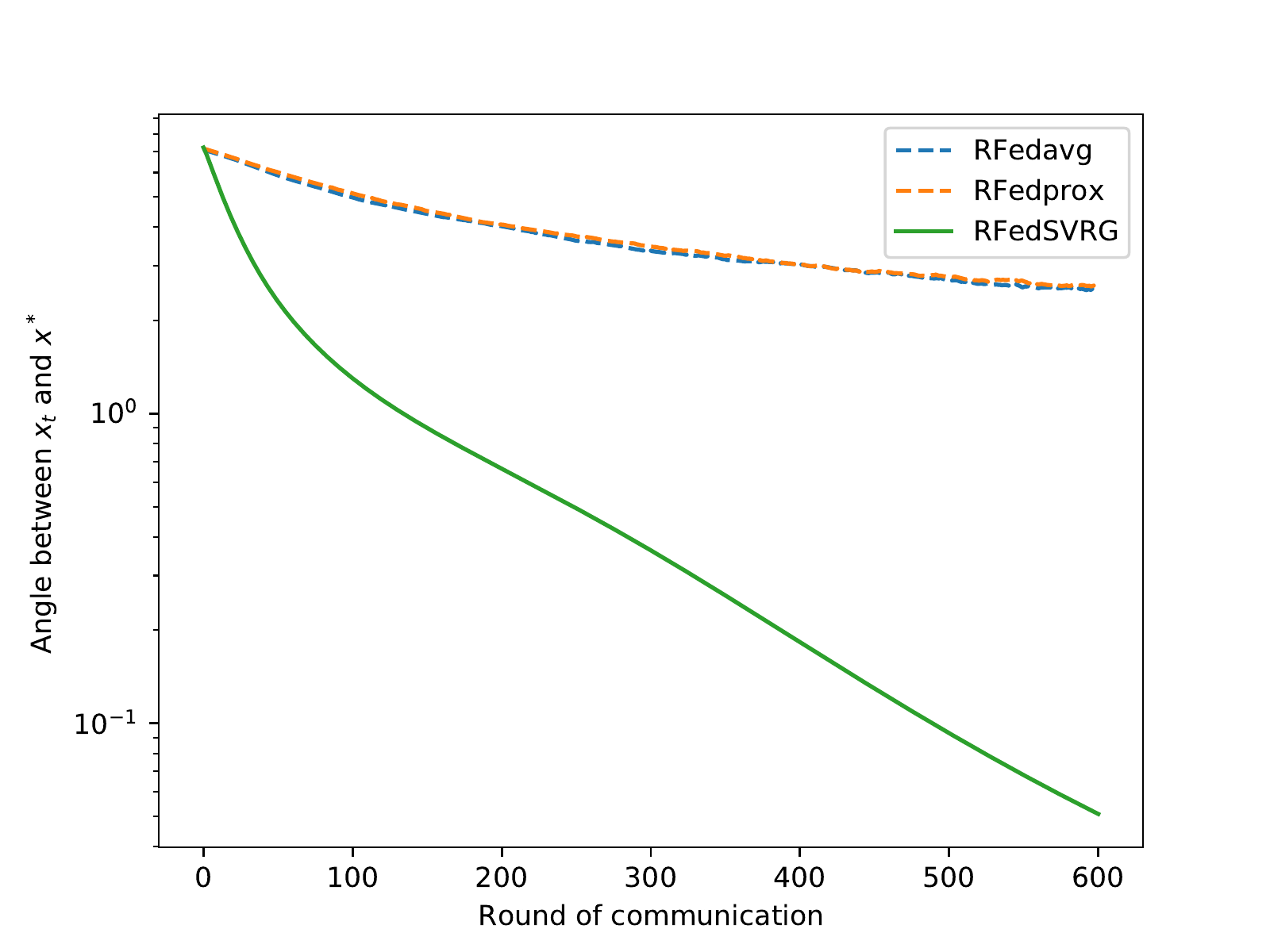}}
    \subfigure[$(n, k)=(500, 50)$]{\includegraphics[width=0.23\textwidth]{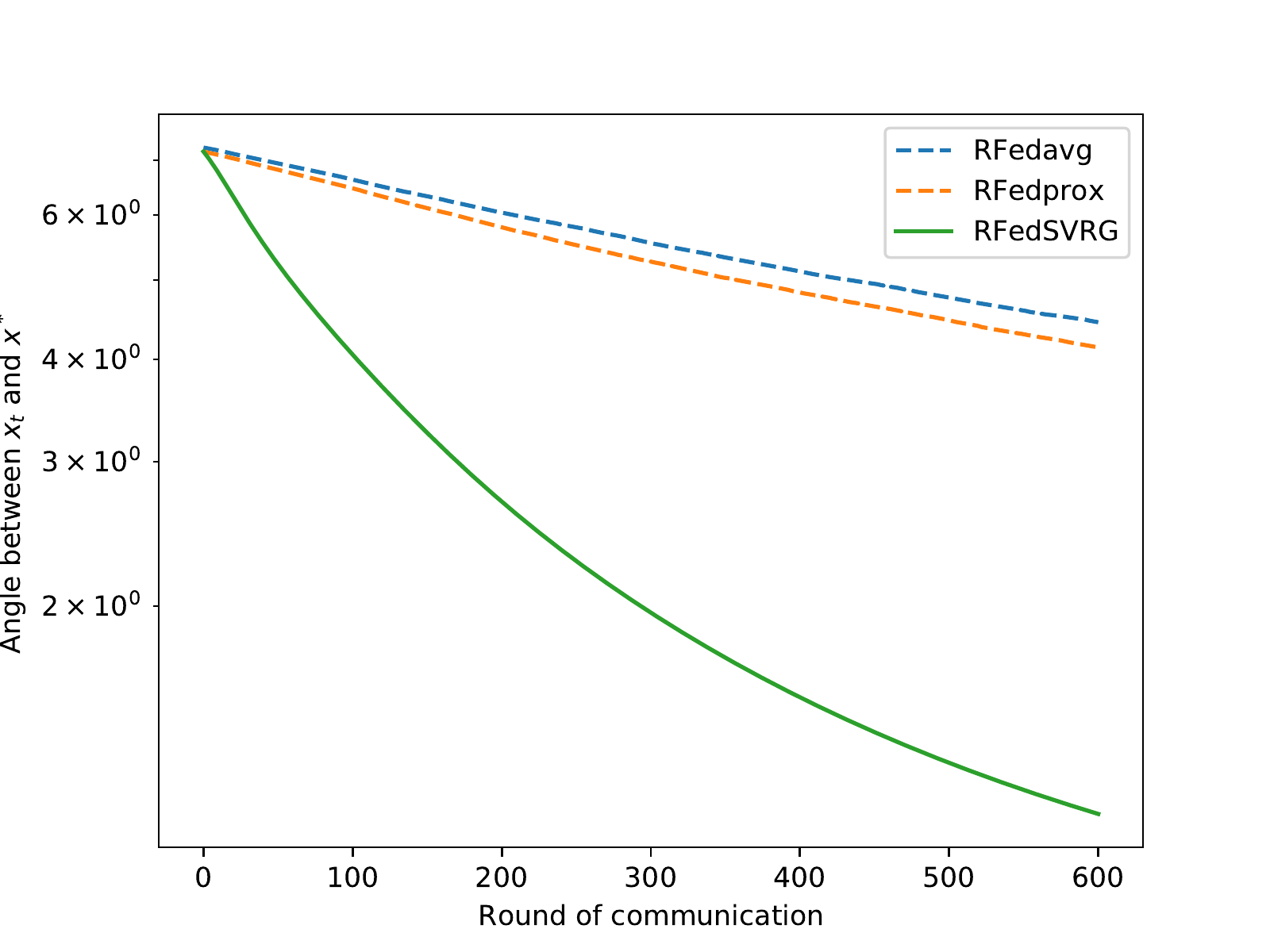}}
    \subfigure[$(n, k)=(1000, 100)$]{\includegraphics[width=0.23\textwidth]{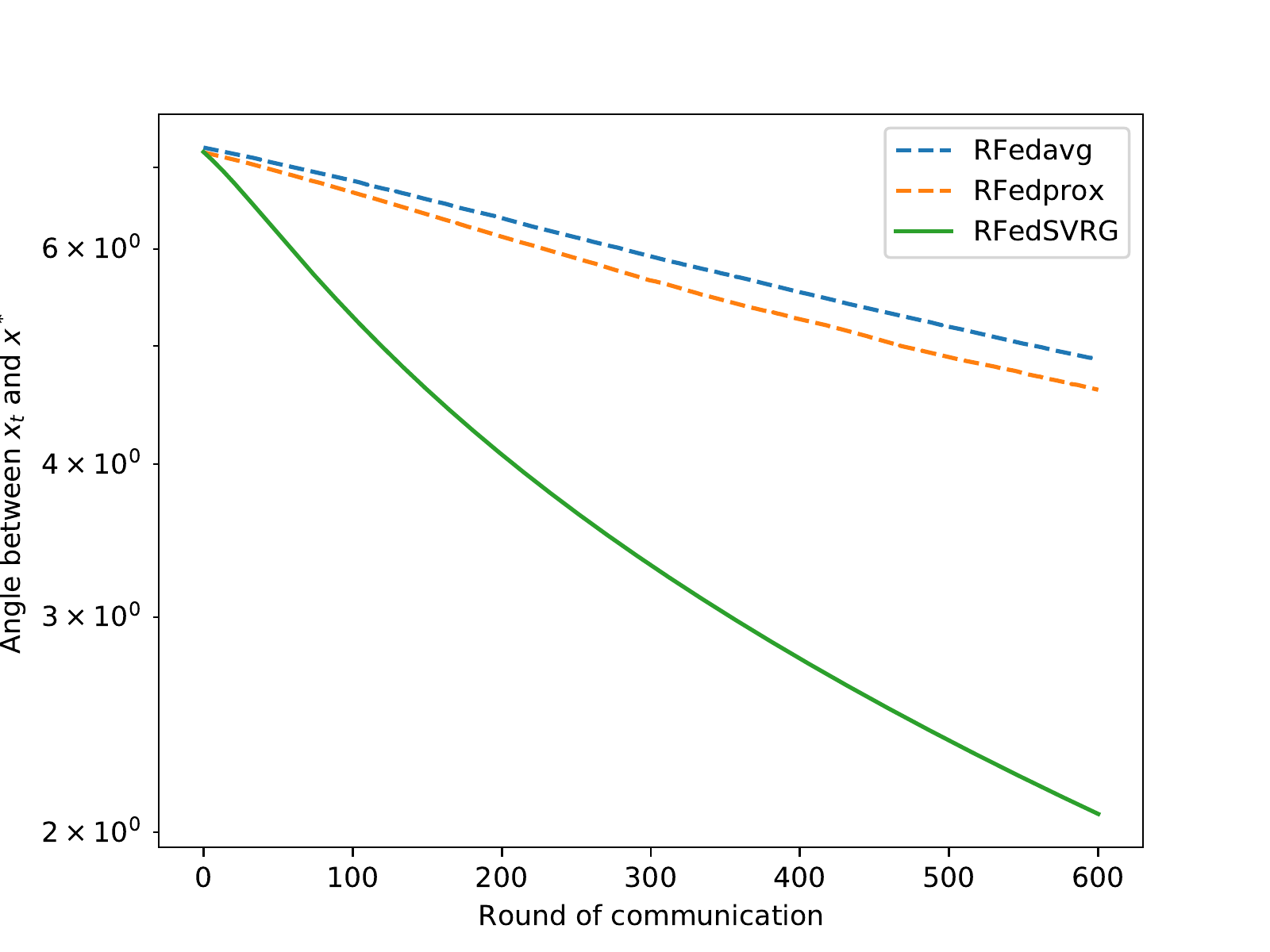}}
    \caption{Results for kPCA. The y-axis of the figures in the first row denotes $\|\grad f(x_t)\|$, and the y-axis of the figures in the second row denotes the principal angle between $x_t$ and $x^*$. The experiments are repeated and averaged over 10 times.}
    \label{fig:kpca_changing_nk}
    \end{center}
\end{figure}

\begin{figure}[t!]
    \begin{center}
    
    \subfigure{\includegraphics[width=0.23\textwidth]{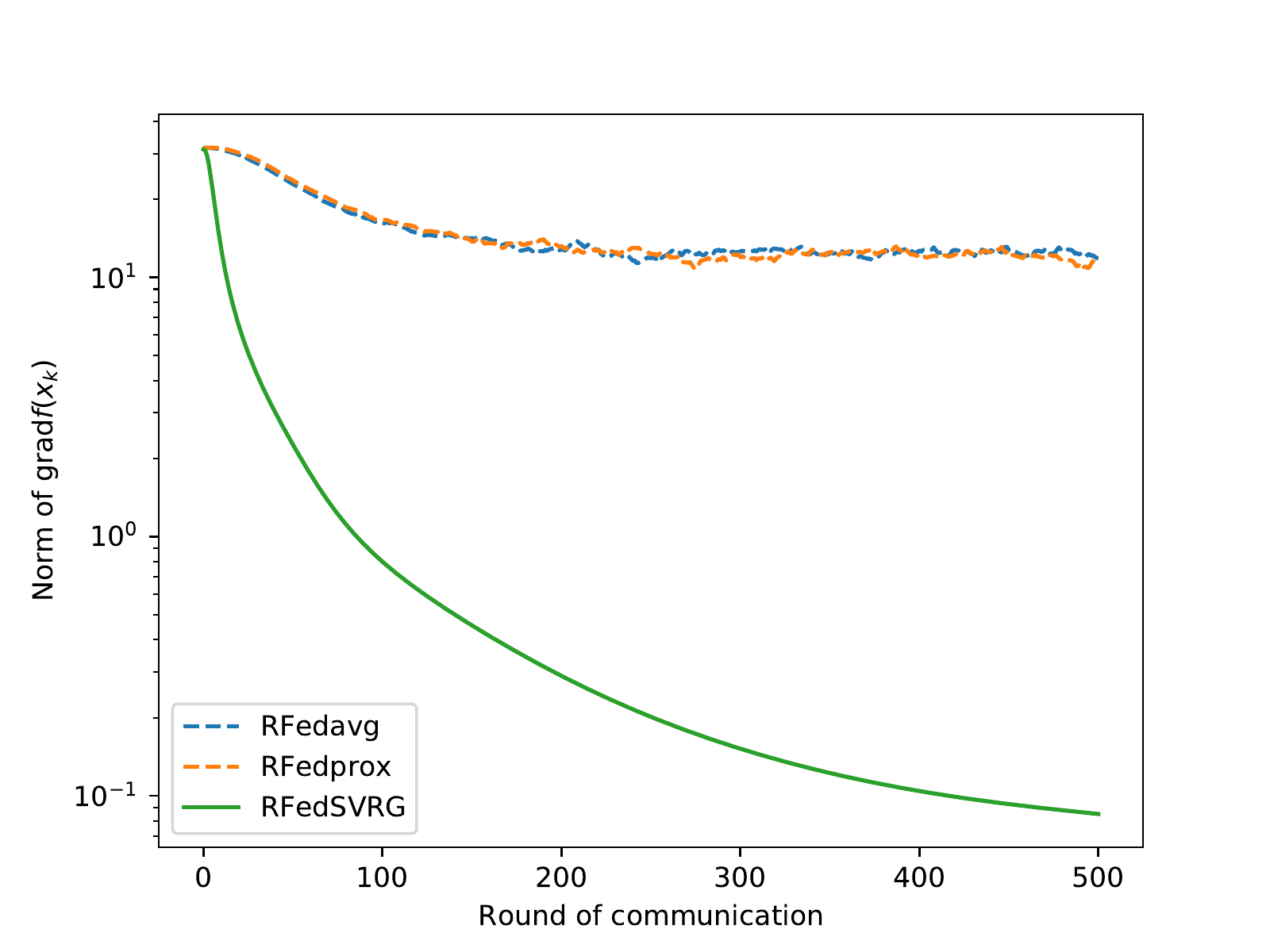}}
    \subfigure{\includegraphics[width=0.23\textwidth]{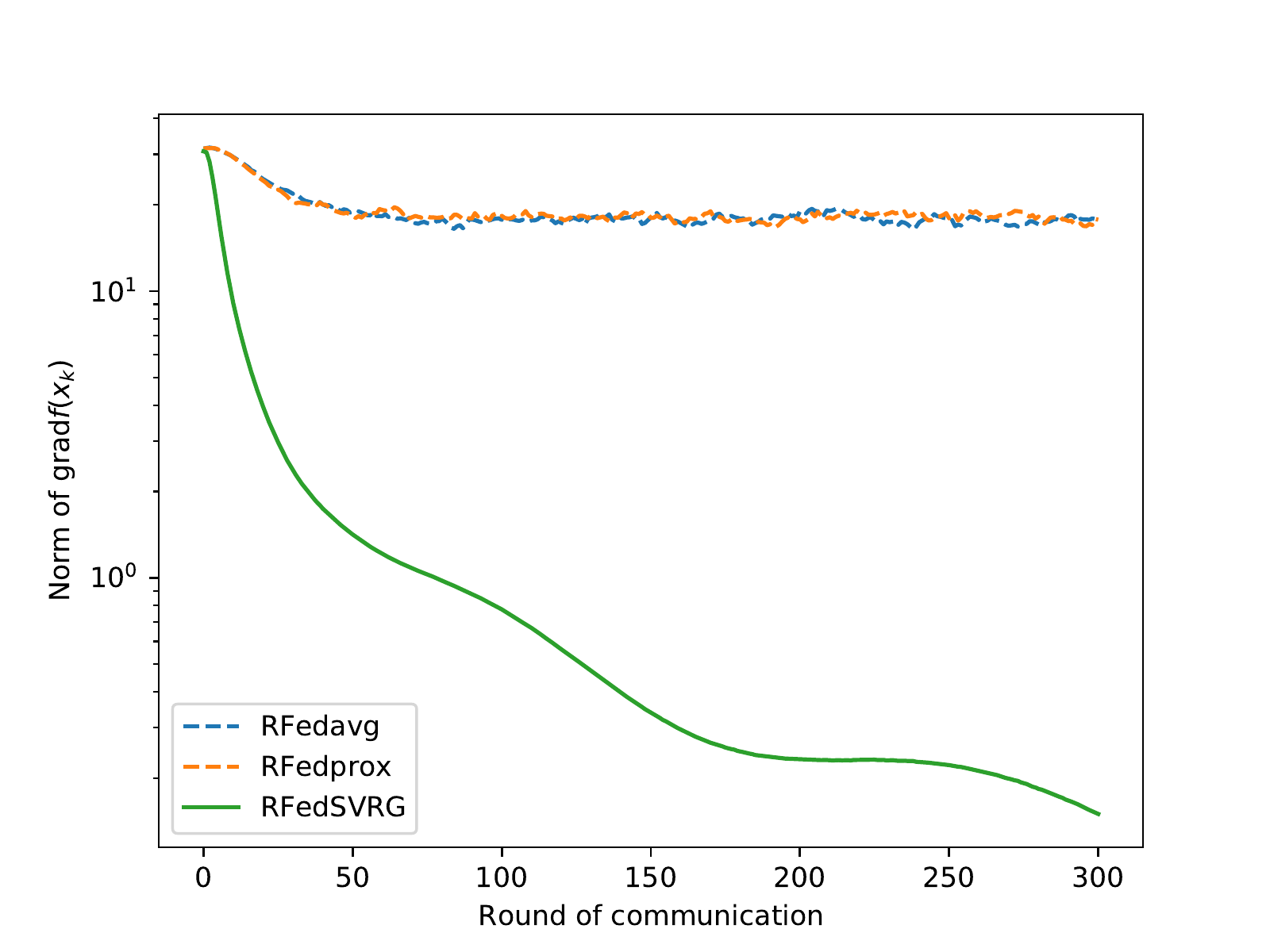}}
    \subfigure{\includegraphics[width=0.23\textwidth]{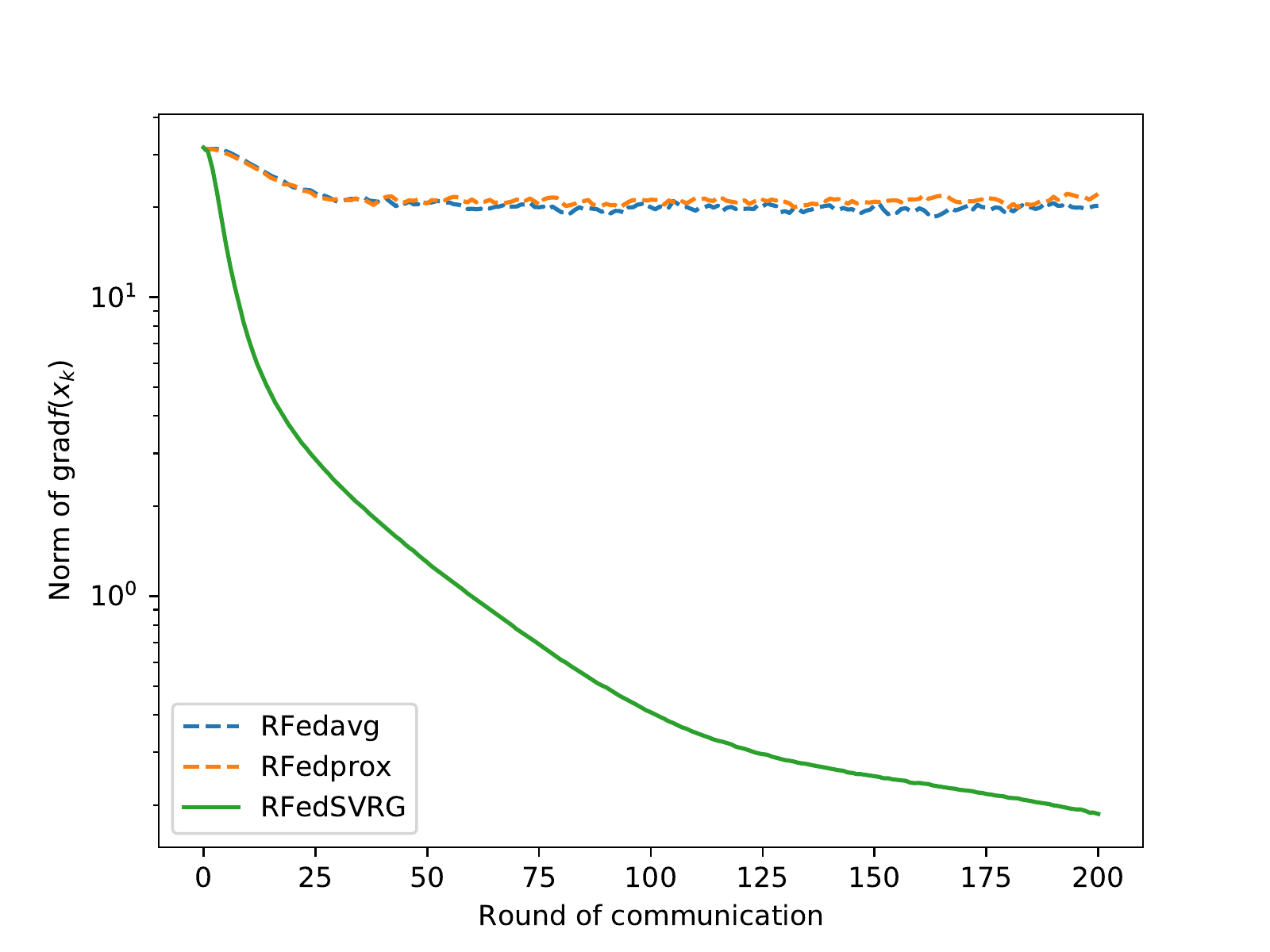}}
    \subfigure{\includegraphics[width=0.23\textwidth]{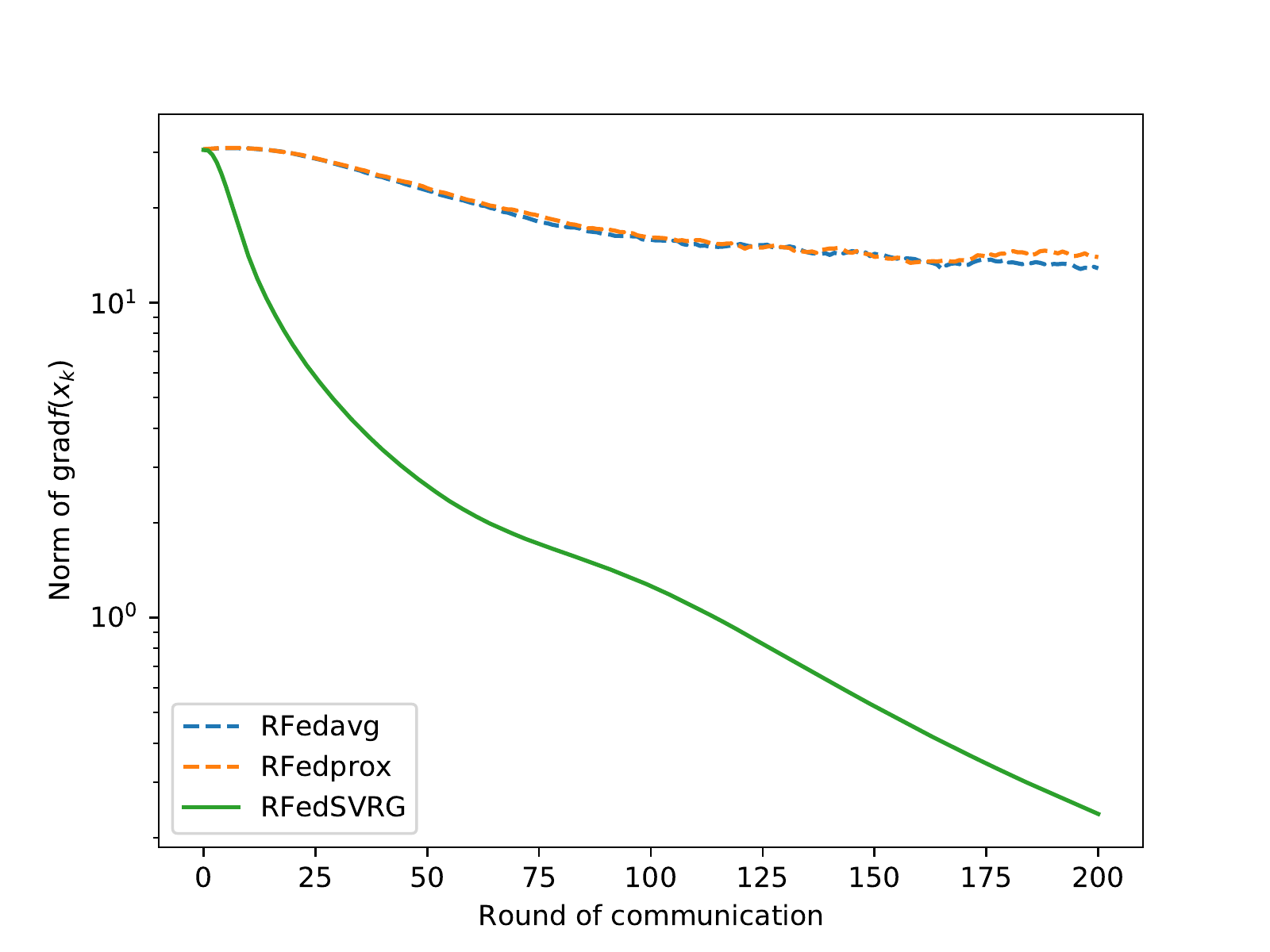}}
    
    \setcounter{subfigure}{0}
    \subfigure[$\tau=1$]{\includegraphics[width=0.23\textwidth]{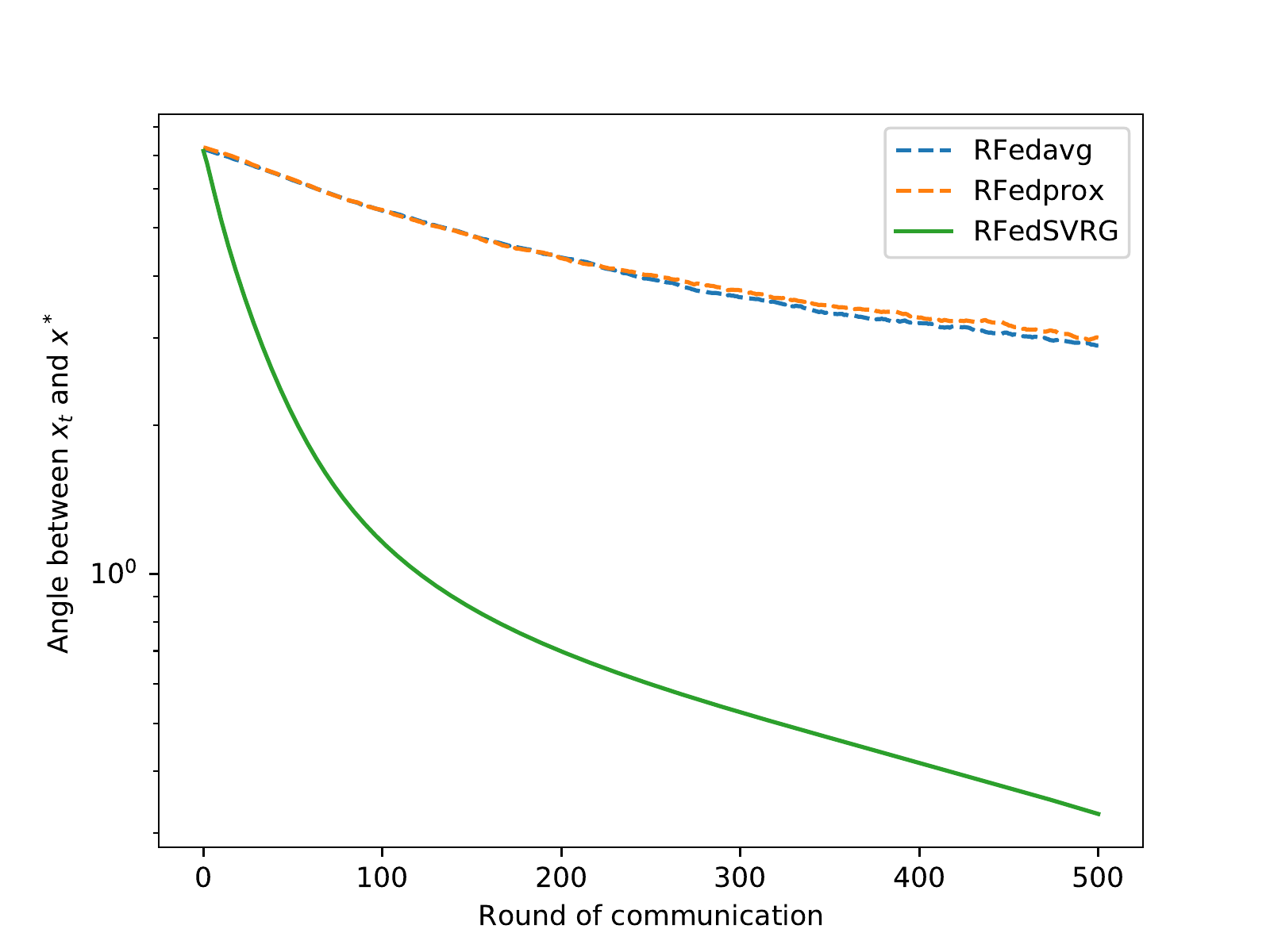}}
    \subfigure[$\tau=10$]{\includegraphics[width=0.23\textwidth]{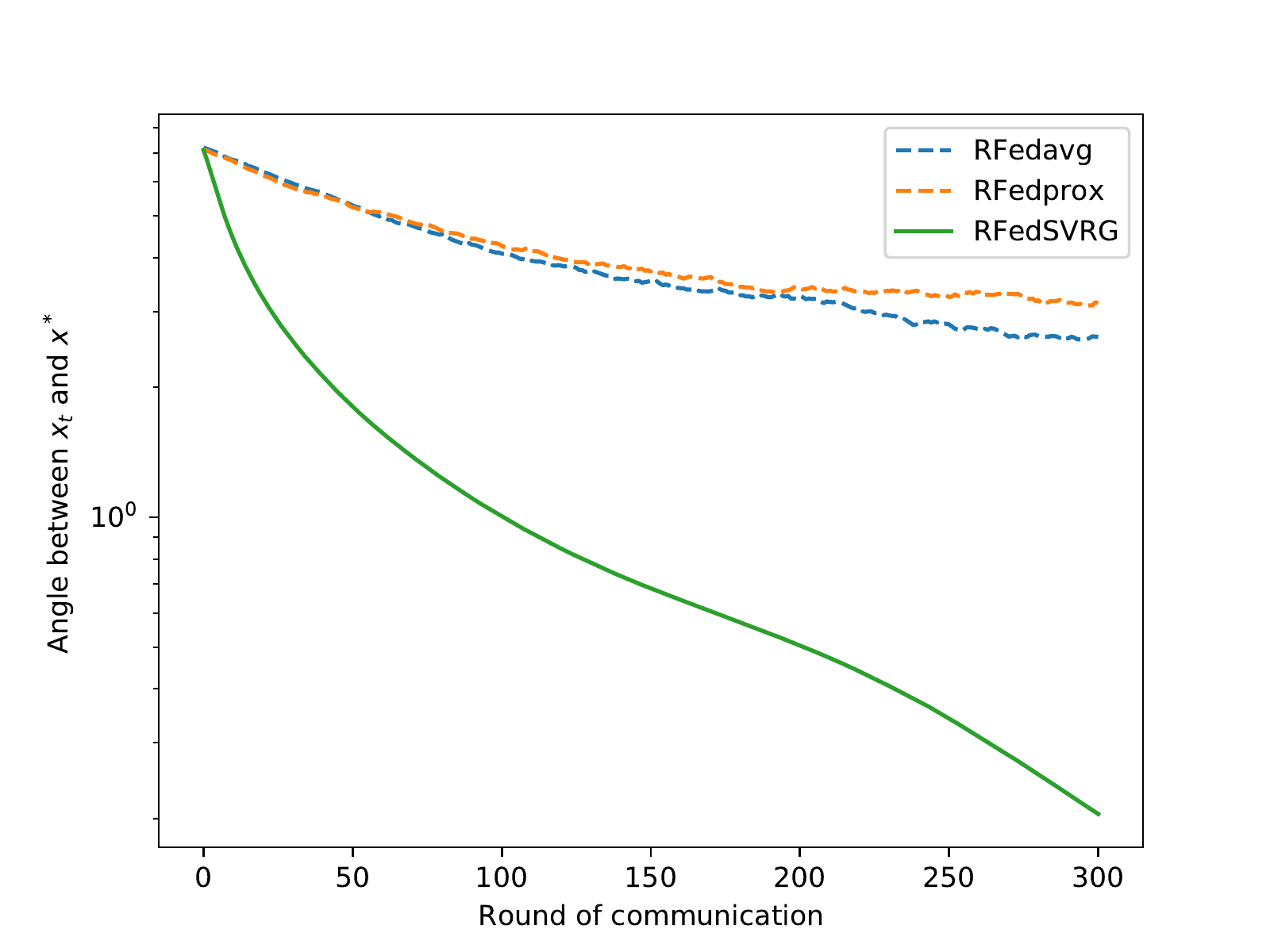}}
    \subfigure[$\tau=50$]{\includegraphics[width=0.23\textwidth]{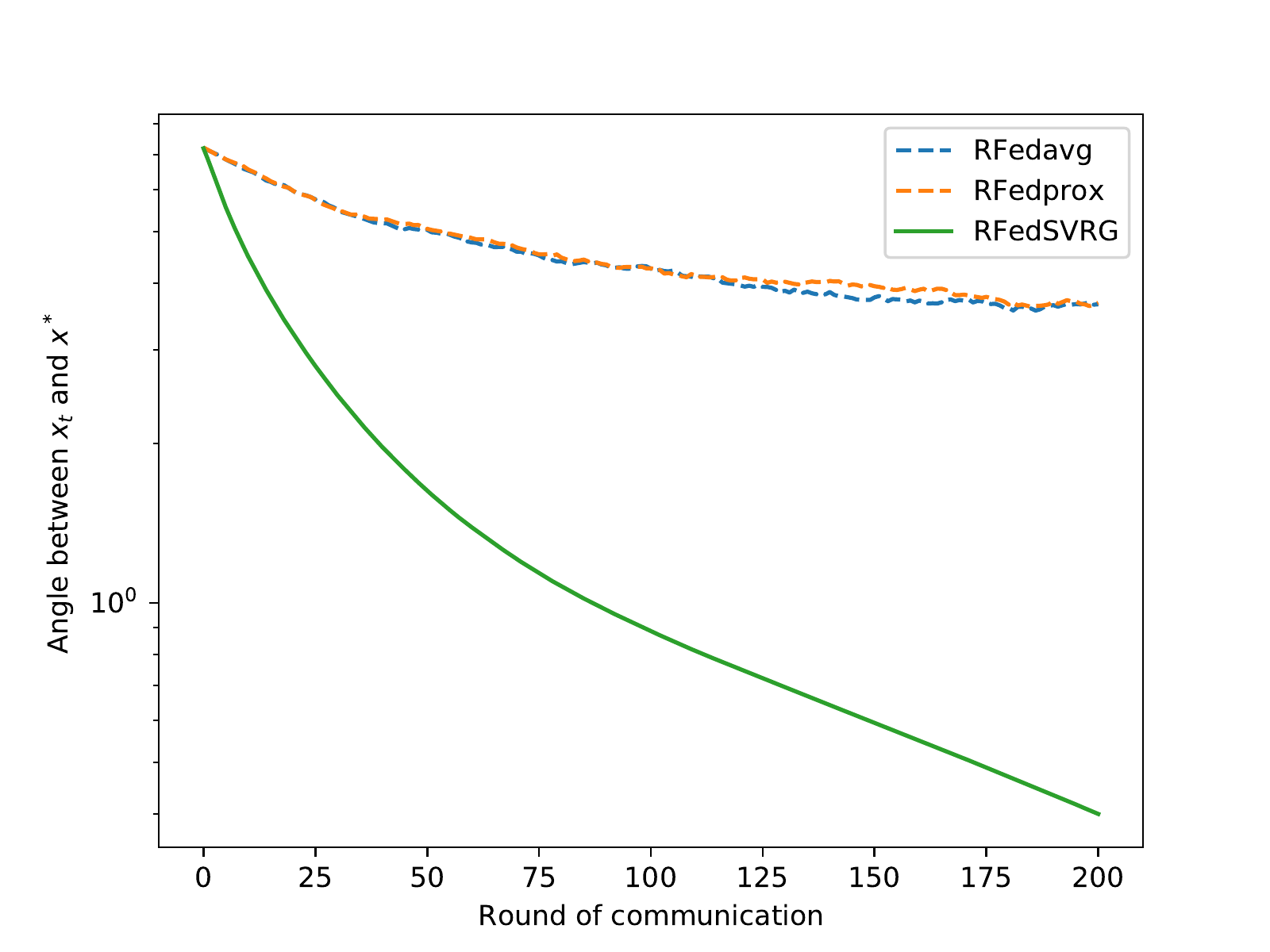}}
    \subfigure[$\tau=100$]{\includegraphics[width=0.23\textwidth]{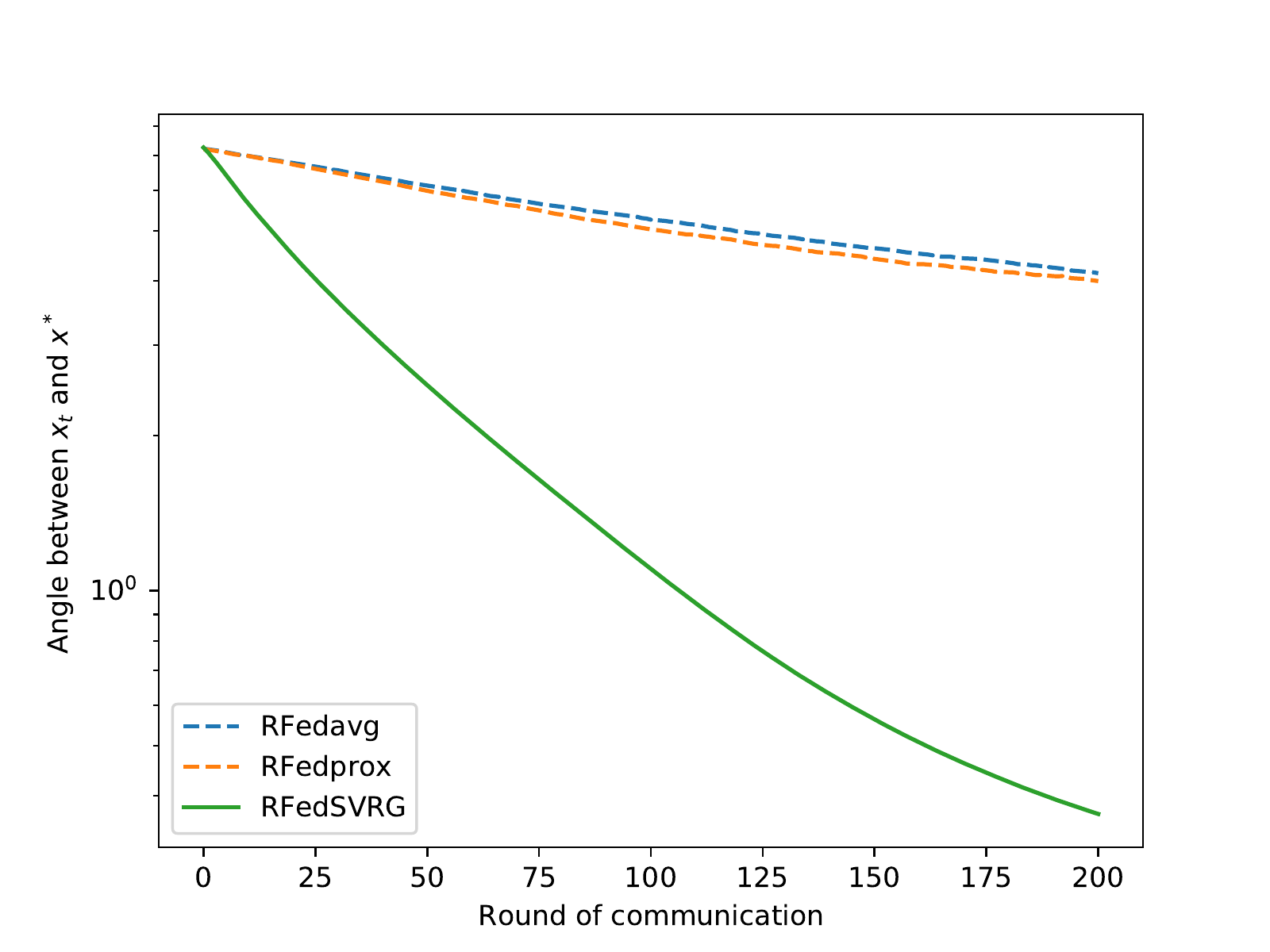}}
    
    \caption{Results for kPCA \eqref{problem_kPCA} with different number of inner loops $\tau=[1, 10, 50, 100]$. The y-axis of the figures in the first row denotes $\|\grad f(x_t)\|$, and the one in the second row denotes the principal angle between $x_t$ and $x^*$. The experiments are repeated and averaged over 10 times.}
    \label{fig:kpca_changing_tau}
    \end{center}
\end{figure}

\subsection{Experiments for kPCA on real data}

We now show the numerical results of the three algorithms on real data. We focus on the kPCA problem \eqref{problem_kPCA} here. We test the three algorithms on three real data sets: the Iris dataset~\cite{forinaextendible}, the wine dataset~\cite{forinaextendible} and the MNIST hand-written dataset~\cite{lecun1998gradient}. For all three datasets, we calculate the first $r$ principal directions and the true optimal loss value directly. We can thus compute the principal angles between the iterate and the ground truth. The experiments are repeated and averaged for 10 random initializations.

For the first two datasets, we randomly partition the datasets into $10$ agents and at each iteration we take $k=5$ agents. The Figures \ref{fig:kpca_iris} and \ref{fig:kpca_wine} show that \texttt{RFedSVRG} is able to effectively decrease the norm of Riemannian gradient and the principal angles while the other two are not as efficient.

\begin{figure}[t!]
    \begin{center}
    \subfigure[Gradient norm]{\includegraphics[width=0.36\textwidth]{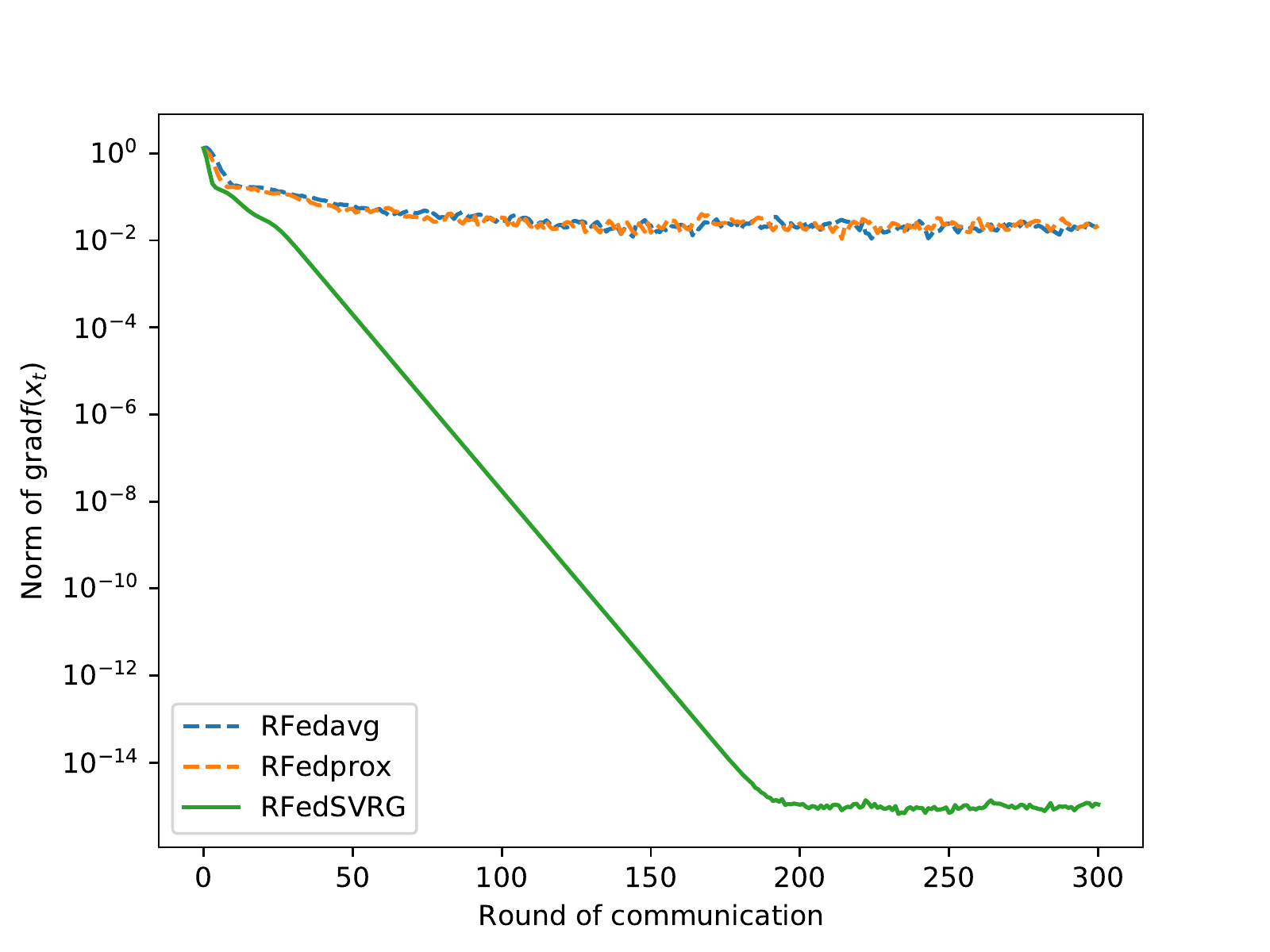}}
    \subfigure[Principal angle]{\includegraphics[width=0.36\textwidth]{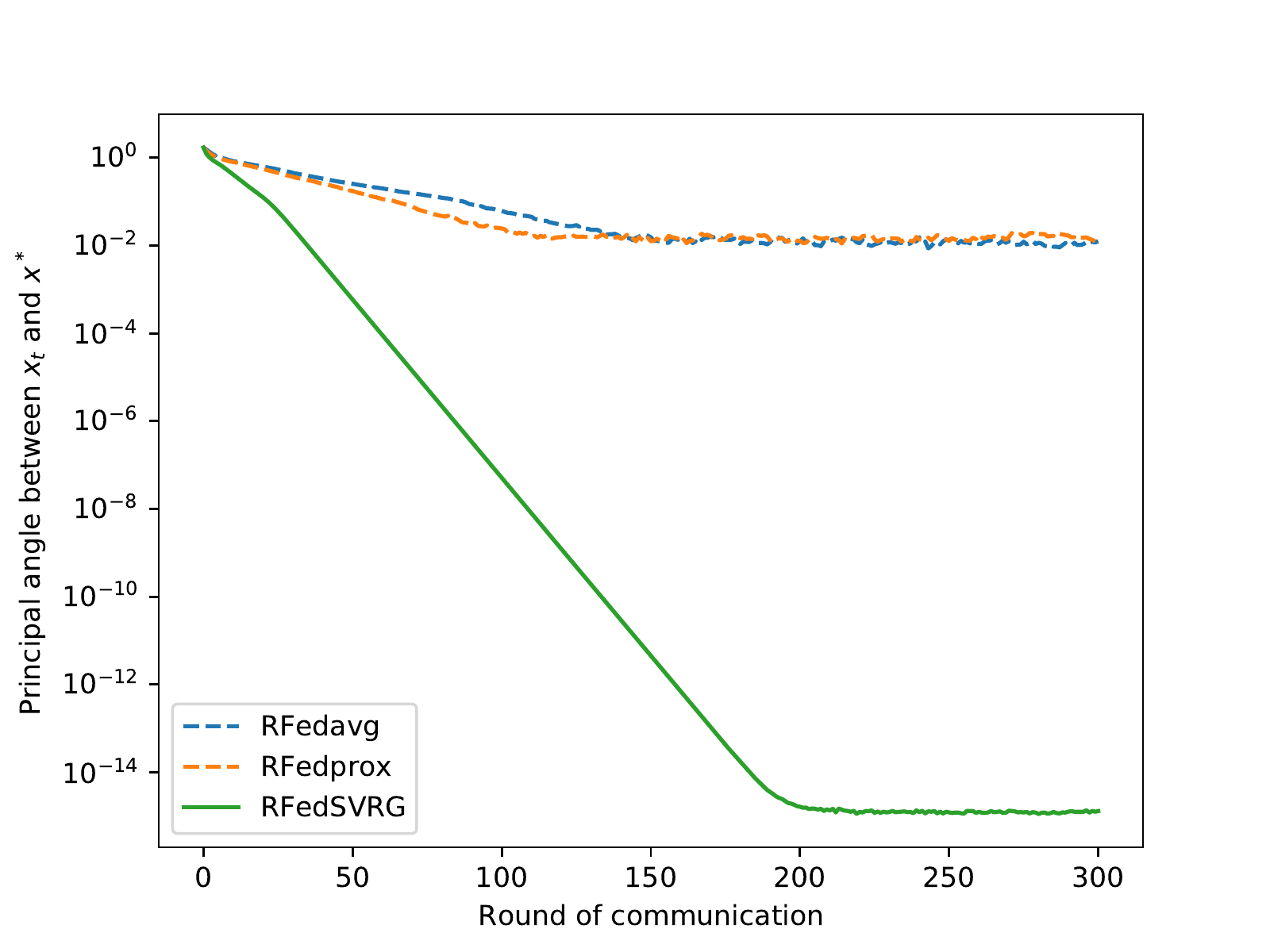}}
    
    \caption{Results for kPCA \eqref{problem_kPCA} on Iris dataset. The data is in $\RR^4$ ($d=4$) and we take $r=2$. The first figure is the norm of Riemannian gradient $\|\grad f(x_t)\|$ and the second is the principal angle between $x_t$ and the true solution $x^*$. }
    \label{fig:kpca_iris}
    \end{center}
\end{figure}

\begin{figure}[t!]
    \begin{center}
    \subfigure[Gradient norm]{\includegraphics[width=0.36\textwidth]{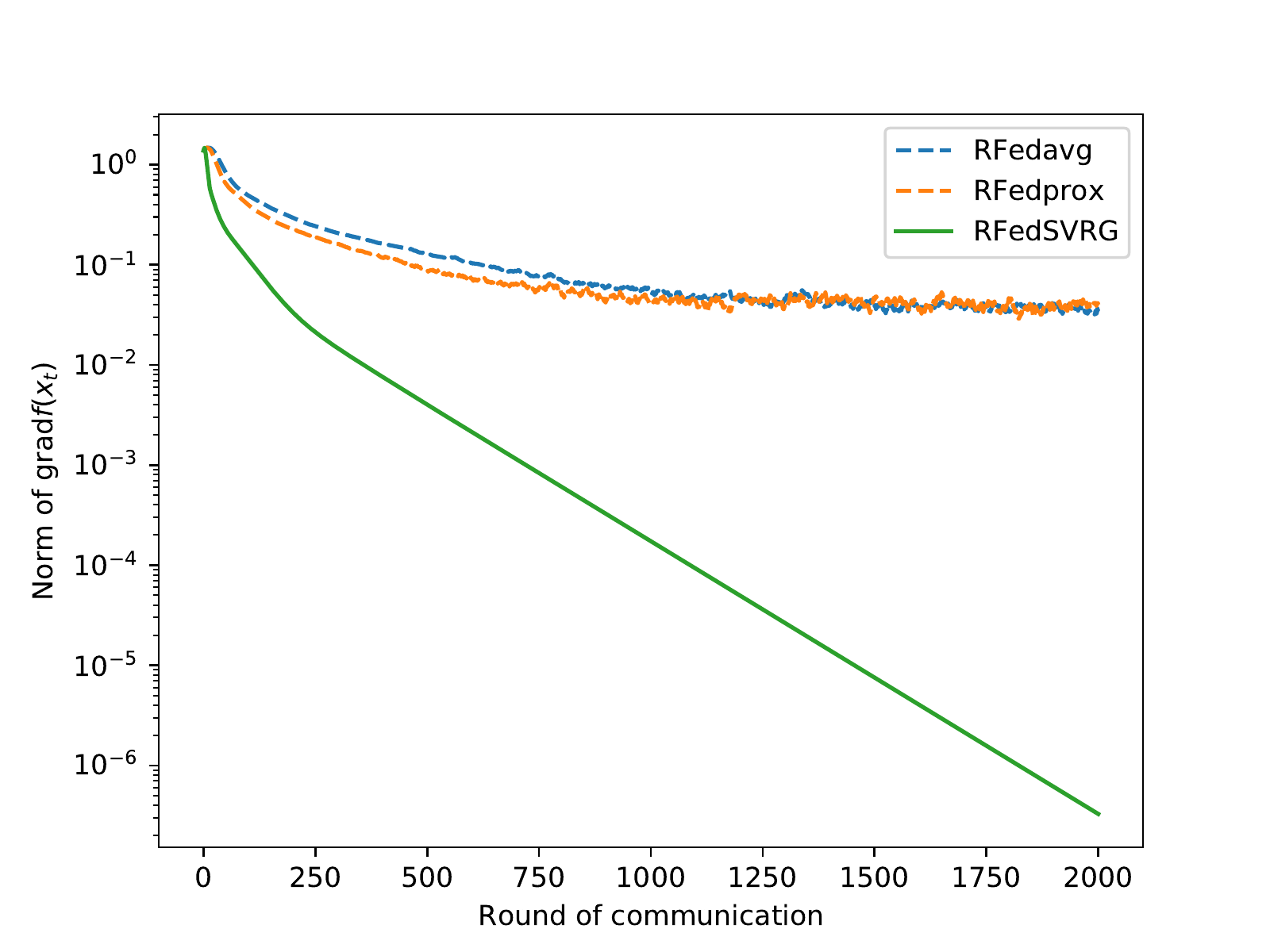}}
    \subfigure[Principal angle]{\includegraphics[width=0.36\textwidth]{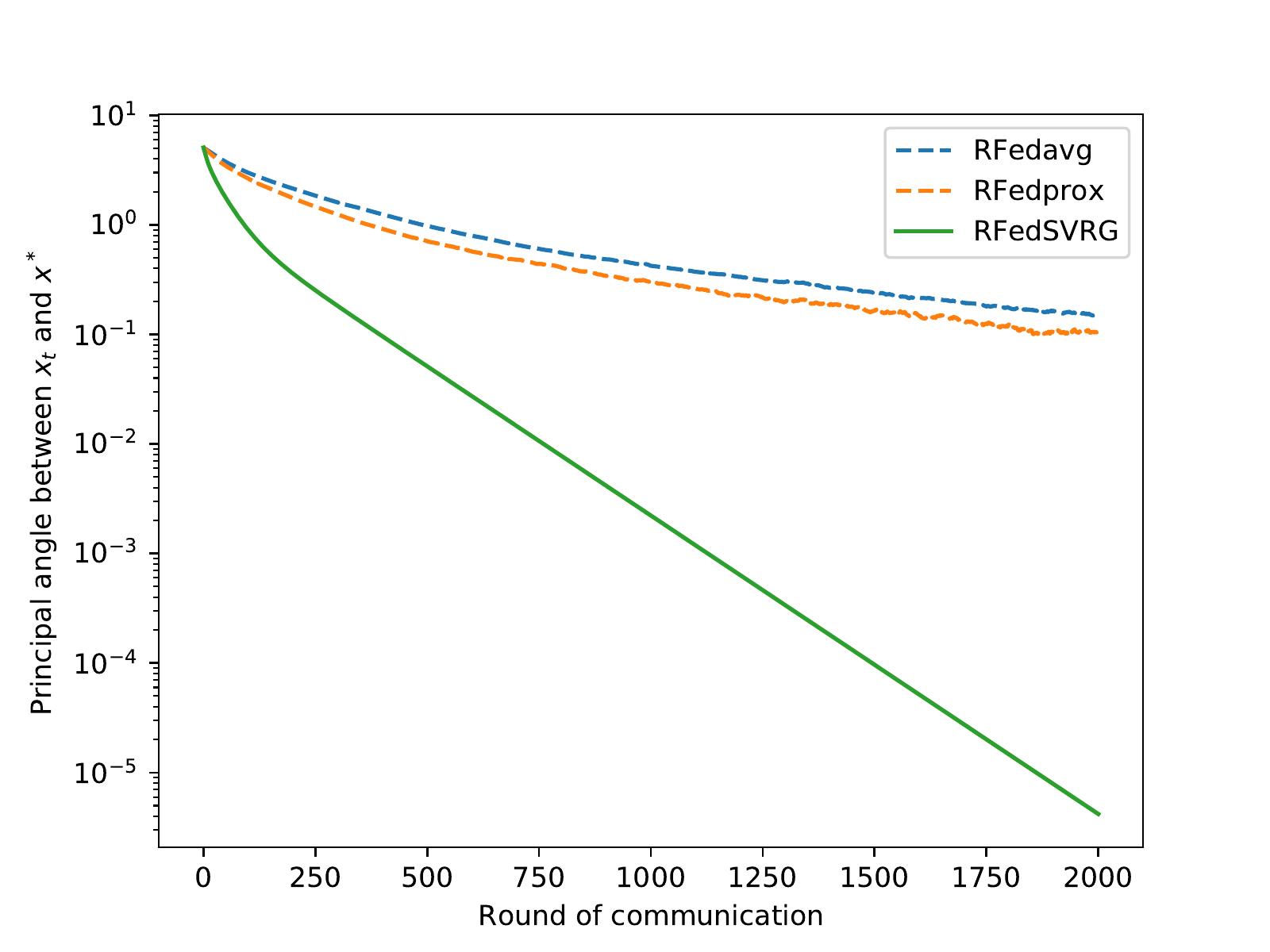}}
    
    \caption{Results for kPCA \eqref{problem_kPCA} with wine dataset. The data is in $\RR^{13}$ ($d=13$) and we take $r=5$. The first figure is the norm of Riemannian gradient $\|\grad f(x_t)\|$ and the second is the principal angle between $x_t$ and the true solution $x^*$.}
    \label{fig:kpca_wine}
    \end{center}
\end{figure}

For the MNIST hand-written dataset, the (training) dataset contains $60000$ hand-written images of size $28\times 28$, i.e. $d=784$. This is a relatively large dataset and we test the proposed algorithms with different number of clients. The results are shown in Figure \ref{fig:kpca_mnist} where the efficiency of \texttt{RFedSVRG} is demonstrated again. The comparison of the two rows of Figure \ref{fig:kpca_mnist} concludes that \texttt{RFedSVRG} shows better efficiency even with a larger number of clients $n$.

\begin{figure}[t!]
    \begin{center}
    \subfigure{\includegraphics[width=0.36\textwidth]{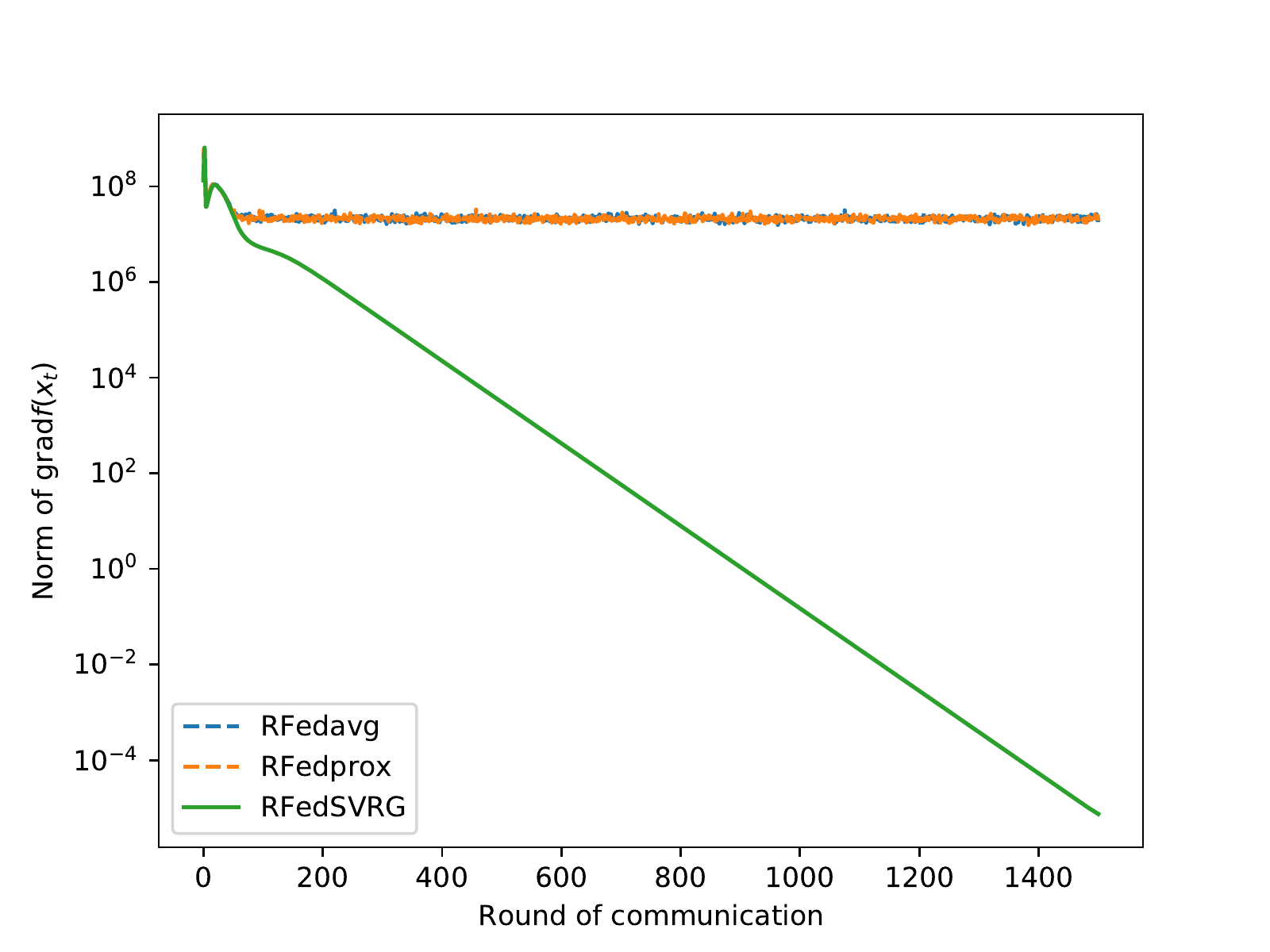}}
    \subfigure{\includegraphics[width=0.36\textwidth]{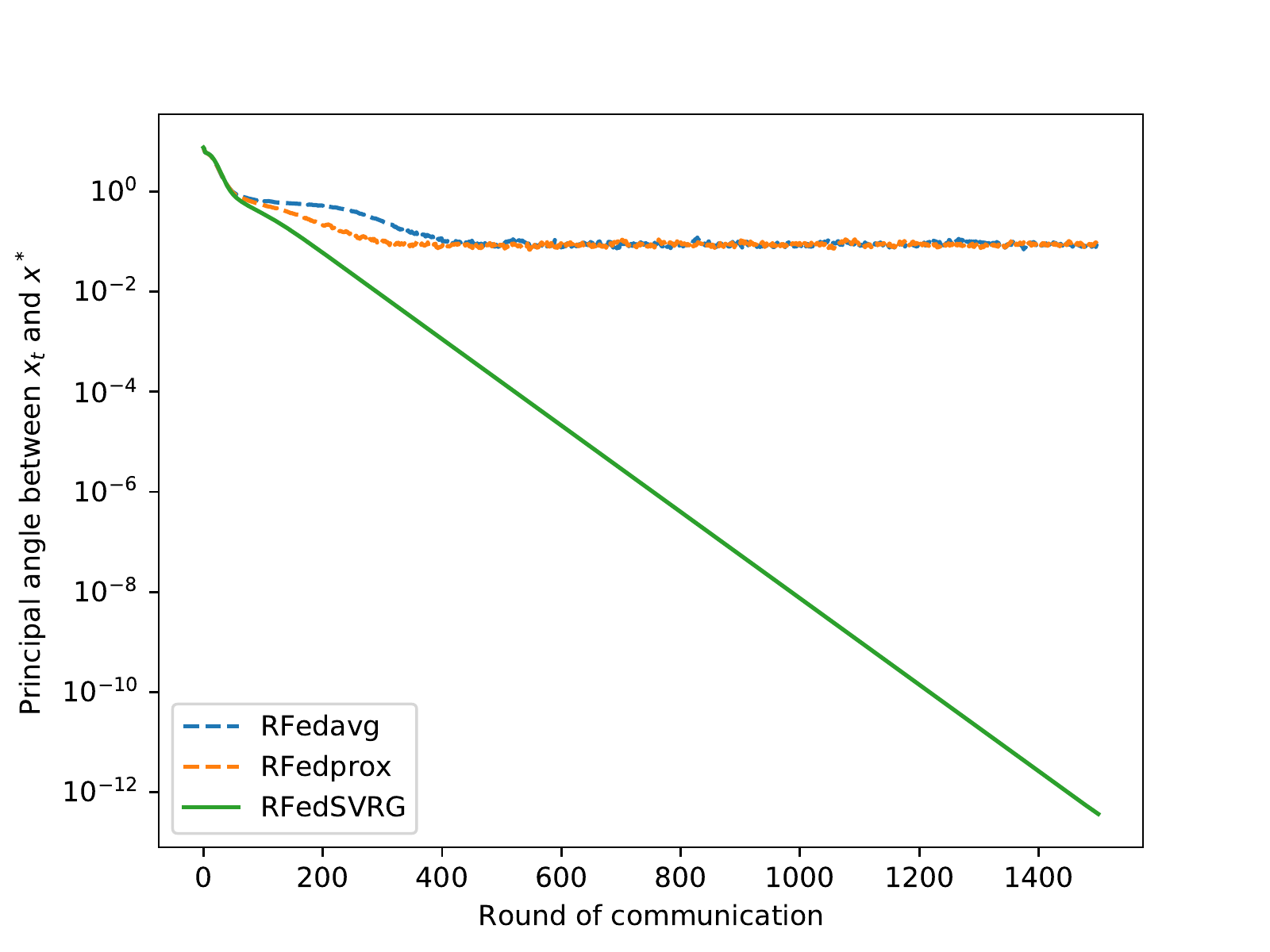}}
    \setcounter{subfigure}{0}
    \subfigure[Gradient norm]{\includegraphics[width=0.36\textwidth]{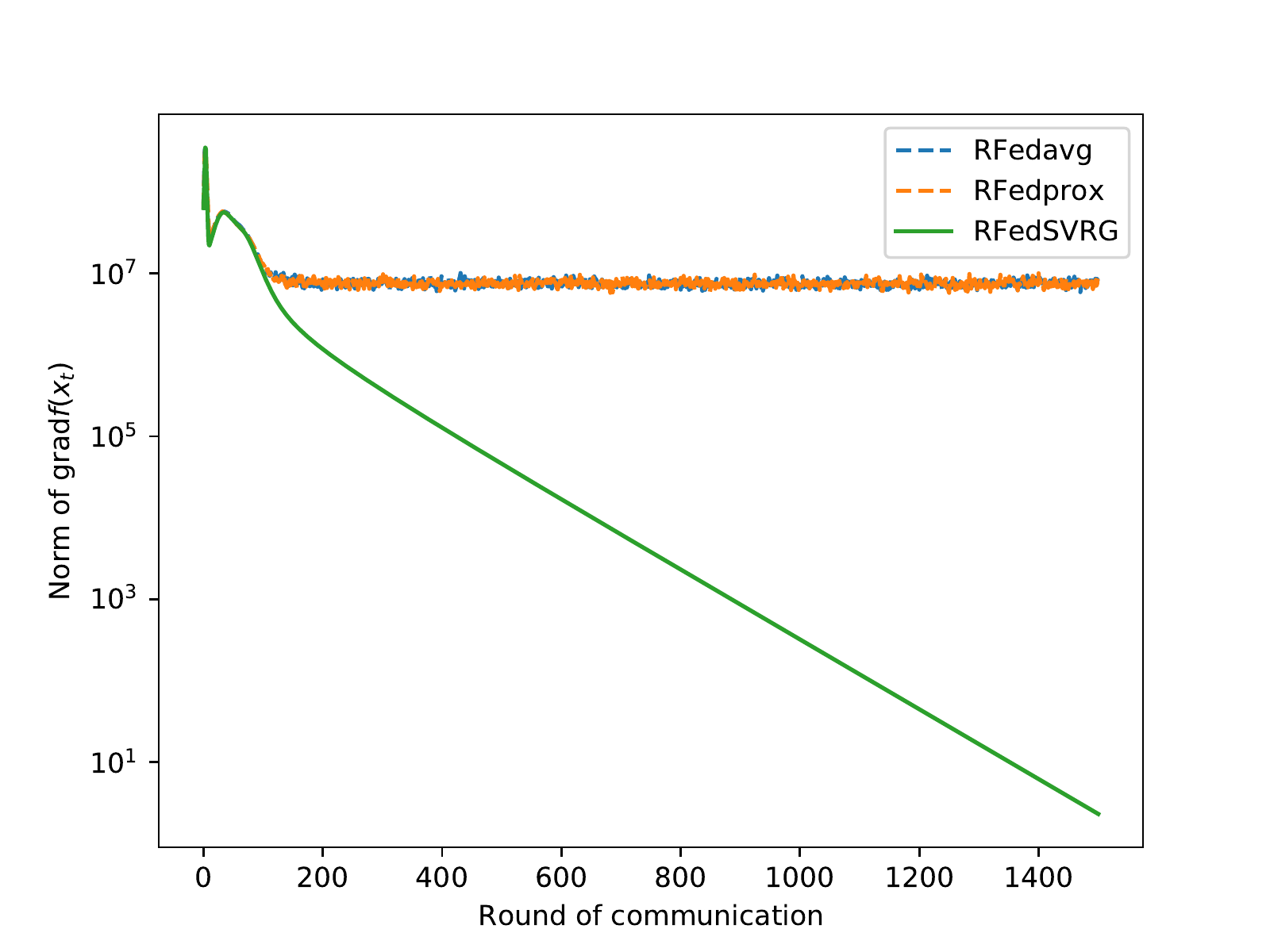}}
    \subfigure[Principal angle]{\includegraphics[width=0.36\textwidth]{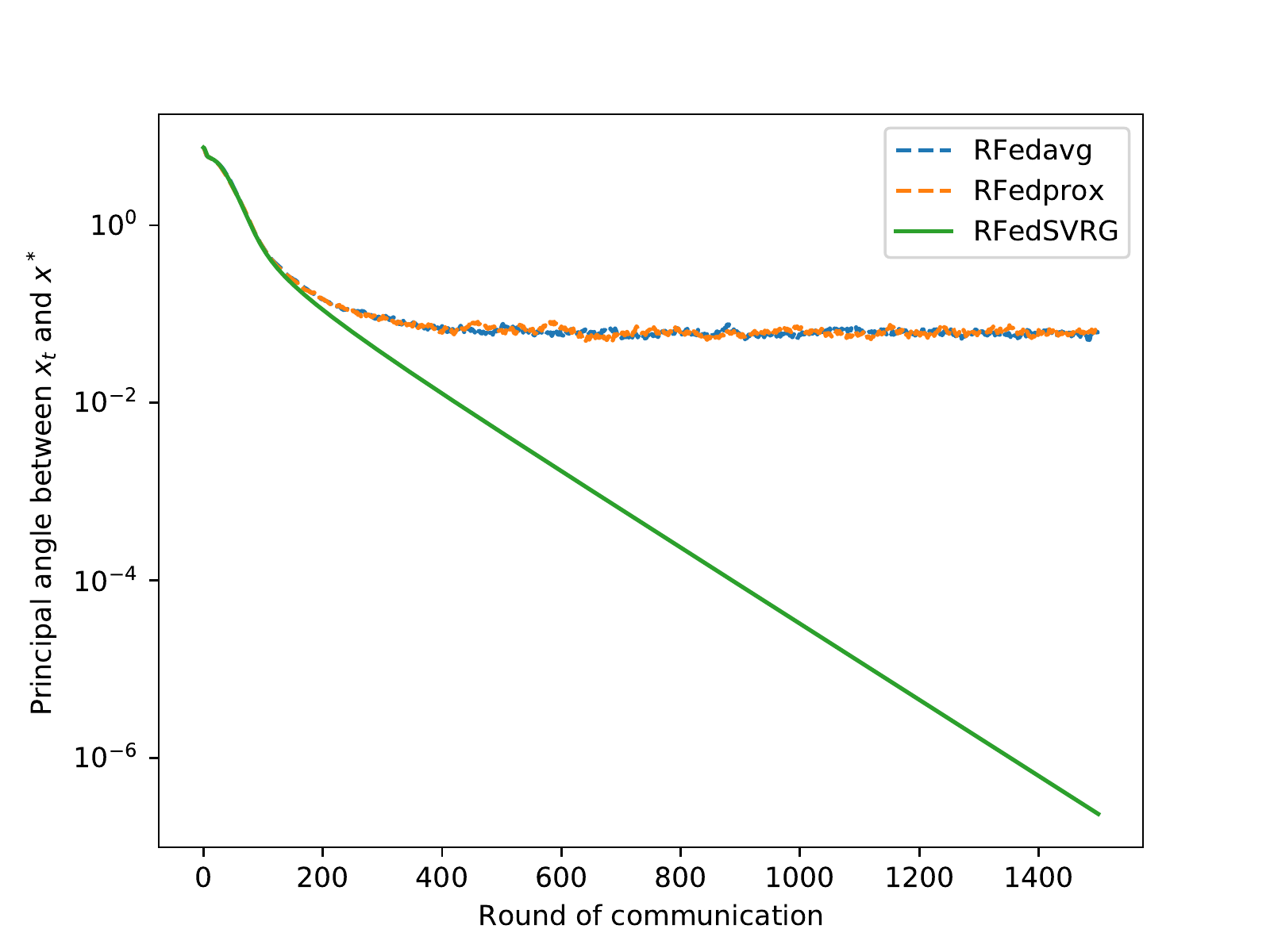}}
    
    \caption{Results for kPCA \eqref{problem_kPCA} with MNIST dataset. The data is in $\RR^{784}$ ($d=784$) and we take $r=5$. The first column is the norm of Riemannian gradient $\grad f(x_t)$ and the second is the principal angle between $x_t$ and the true solution $x^*$. The two rows corresponds to $n=100$ and $n=200$. We take $k=n/10$ and $\tau=5$ for all algorithms.}
    \label{fig:kpca_mnist}
    \end{center}
\end{figure}

\section{Conclusions}

In this paper, we studied the federated optimization over Riemannian manifolds. We proposed a Riemannian federated SVRG algorithm and analyzed its convergence rate to an $\epsilon$-stationary point. To the best of our knowledge, this is the first federated algorithm over Riemannian manifolds with convergence guarantees. Numerical experiments on federated PCA and federated kPCA were conducted to demonstrate the efficiency of the proposed method. {Developing algorithms with lower communication cost, better scalability and sparse solutions are some important topics for future research.}



\newpage
{\small
\bibliography{bibfile}

\begin{thebibliography}{10}

\bibitem{absil2009optimization}
P-A Absil, Robert Mahony, and Rodolphe Sepulchre.
\newblock Optimization algorithms on matrix manifolds.
\newblock In {\em Optimization Algorithms on Matrix Manifolds}. Princeton
  University Press, 2009.

\bibitem{absil2012projection}
P-A Absil and J{\'e}r{\^o}me Malick.
\newblock Projection-like retractions on matrix manifolds.
\newblock {\em SIAM Journal on Optimization}, 22(1):135--158, 2012.

\bibitem{alimisis2021distributed}
Foivos Alimisis, Peter Davies, Bart Vandereycken, and Dan Alistarh.
\newblock Distributed principal component analysis with limited communication.
\newblock {\em Advances in Neural Information Processing Systems}, 34, 2021.

\bibitem{manopt}
N.~Boumal, B.~Mishra, P.-A. Absil, and R.~Sepulchre.
\newblock {M}anopt, a {M}atlab toolbox for optimization on manifolds.
\newblock {\em Journal of Machine Learning Research}, 15(42):1455--1459, 2014.

\bibitem{boumal2022intromanifolds}
Nicolas Boumal.
\newblock An introduction to optimization on smooth manifolds.
\newblock To appear with Cambridge University Press, Jan 2022.

\bibitem{boumal2018global}
Nicolas Boumal, Pierre-Antoine Absil, and Coralia Cartis.
\newblock Global rates of convergence for nonconvex optimization on manifolds.
\newblock {\em IMA Journal of Numerical Analysis}, 39(1):1--33, 2018.

\bibitem{charles2021convergence}
Zachary Charles and Jakub Kone{\v{c}}n{\`y}.
\newblock Convergence and accuracy trade-offs in federated learning and
  meta-learning.
\newblock In {\em International Conference on Artificial Intelligence and
  Statistics}, pages 2575--2583. PMLR, 2021.

\bibitem{chen2021decentralized}
Shixiang Chen, Alfredo Garcia, Mingyi Hong, and Shahin Shahrampour.
\newblock Decentralized riemannian gradient descent on the stiefel manifold.
\newblock In {\em International Conference on Machine Learning}, pages
  1594--1605. PMLR, 2021.

\bibitem{chen2021local}
Shixiang Chen, Alfredo Garcia, Mingyi Hong, and Shahin Shahrampour.
\newblock On the local linear rate of consensus on the stiefel manifold.
\newblock {\em arXiv preprint arXiv:2101.09346}, 2021.

\bibitem{forinaextendible}
Michele Forina, Riccardo Leardi, Armanino C, and Sergio Lanteri.
\newblock {\em PARVUS: An Extendable Package of Programs for Data Exploration}.
\newblock Elsevier, Amsterdam, 01 1998.

\bibitem{grammenos2020federated}
Andreas Grammenos, Rodrigo Mendoza~Smith, Jon Crowcroft, and Cecilia Mascolo.
\newblock Federated principal component analysis.
\newblock {\em Advances in Neural Information Processing Systems},
  33:6453--6464, 2020.

\bibitem{kairouz2021advances}
Peter Kairouz, H~Brendan McMahan, Brendan Avent, Aur{\'e}lien Bellet, Mehdi
  Bennis, Arjun~Nitin Bhagoji, Kallista Bonawitz, Zachary Charles, Graham
  Cormode, Rachel Cummings, et~al.
\newblock Advances and open problems in federated learning.
\newblock {\em Foundations and Trends{\textregistered} in Machine Learning},
  14(1--2):1--210, 2021.

\bibitem{kaneko2012empirical}
Tetsuya Kaneko, Simone Fiori, and Toshihisa Tanaka.
\newblock Empirical arithmetic averaging over the compact stiefel manifold.
\newblock {\em IEEE Transactions on Signal Processing}, 61(4):883--894, 2012.

\bibitem{karimireddy2020scaffold}
Sai~Praneeth Karimireddy, Satyen Kale, Mehryar Mohri, Sashank Reddi, Sebastian
  Stich, and Ananda~Theertha Suresh.
\newblock Scaffold: Stochastic controlled averaging for federated learning.
\newblock In {\em International Conference on Machine Learning}, pages
  5132--5143. PMLR, 2020.

\bibitem{knyazev2012principal}
Andrew~V Knyazev and Peizhen Zhu.
\newblock Principal angles between subspaces and their tangents.
\newblock {\em arXiv preprint arXiv:1209.0523}, 2012.

\bibitem{konevcny2016federated}
Jakub Kone{\v{c}}n{\`y}, H~Brendan McMahan, Daniel Ramage, and Peter
  Richt{\'a}rik.
\newblock Federated optimization: Distributed machine learning for on-device
  intelligence.
\newblock {\em arXiv preprint arXiv:1610.02527}, 2016.

\bibitem{lecun1998gradient}
Yann LeCun, L{\'e}on Bottou, Yoshua Bengio, and Patrick Haffner.
\newblock Gradient-based learning applied to document recognition.
\newblock {\em Proceedings of the IEEE}, 86(11):2278--2324, 1998.

\bibitem{lee2006riemannian}
John~M Lee.
\newblock {\em Riemannian manifolds: an introduction to curvature}, volume 176.
\newblock Springer Science \& Business Media, 2006.

\bibitem{li2020federated}
Tian Li, Anit~Kumar Sahu, Manzil Zaheer, Maziar Sanjabi, Ameet Talwalkar, and
  Virginia Smith.
\newblock Federated optimization in heterogeneous networks.
\newblock {\em Proceedings of Machine Learning and Systems}, 2:429--450, 2020.

\bibitem{li2019convergence}
Xiang Li, Kaixuan Huang, Wenhao Yang, Shusen Wang, and Zhihua Zhang.
\newblock On the convergence of fedavg on non-iid data.
\newblock {\em arXiv preprint arXiv:1907.02189}, 2019.

\bibitem{malinovskiy2020local}
Grigory Malinovskiy, Dmitry Kovalev, Elnur Gasanov, Laurent Condat, and Peter
  Richtarik.
\newblock From local sgd to local fixed-point methods for federated learning.
\newblock In {\em International Conference on Machine Learning}, pages
  6692--6701. PMLR, 2020.

\bibitem{mcmahan2017communication}
Brendan McMahan, Eider Moore, Daniel Ramage, Seth Hampson, and Blaise~Aguera
  y~Arcas.
\newblock Communication-efficient learning of deep networks from decentralized
  data.
\newblock In {\em Artificial intelligence and statistics}, pages 1273--1282.
  PMLR, 2017.

\bibitem{mitra2021linear}
Aritra Mitra, Rayana Jaafar, George~J Pappas, and Hamed Hassani.
\newblock Linear convergence in federated learning: Tackling client
  heterogeneity and sparse gradients.
\newblock {\em Advances in Neural Information Processing Systems},
  34:14606--14619, 2021.

\bibitem{pathak2020fedsplit}
Reese Pathak and Martin~J Wainwright.
\newblock Fedsplit: An algorithmic framework for fast federated optimization.
\newblock {\em Advances in Neural Information Processing Systems},
  33:7057--7066, 2020.

\bibitem{shah2017distributed}
Suhail~M Shah.
\newblock Distributed optimization on riemannian manifolds for multi-agent
  networks.
\newblock {\em arXiv preprint arXiv:1711.11196}, 2017.

\bibitem{pymanopt}
J.~Townsend, N.~Koep, and S.~Weichwald.
\newblock {P}y{M}anopt: a {P}ython toolbox for optimization on manifolds using
  automatic differentiation.
\newblock {\em Journal of Machine Learning Research}, 17(137):1--5, 2016.

\bibitem{tron2012riemannian}
Roberto Tron, Bijan Afsari, and Ren{\'e} Vidal.
\newblock Riemannian consensus for manifolds with bounded curvature.
\newblock {\em IEEE Transactions on Automatic Control}, 58(4):921--934, 2012.

\bibitem{Tu2011manifolds}
Loring~W Tu.
\newblock {\em An Introduction to Manifolds}.
\newblock Springer Science \& Universitext, 2011.

\bibitem{wang2020tackling}
Jianyu Wang, Qinghua Liu, Hao Liang, Gauri Joshi, and H~Vincent Poor.
\newblock Tackling the objective inconsistency problem in heterogeneous
  federated optimization.
\newblock {\em Advances in neural information processing systems},
  33:7611--7623, 2020.

\bibitem{zhang2016fast}
Hongyi Zhang, Sashank J~Reddi, and Suvrit Sra.
\newblock Riemannian svrg: Fast stochastic optimization on riemannian
  manifolds.
\newblock {\em Advances in Neural Information Processing Systems}, 29, 2016.

\bibitem{zhang2016first}
Hongyi Zhang and Suvrit Sra.
\newblock First-order methods for geodesically convex optimization.
\newblock In {\em Conference on Learning Theory}, pages 1617--1638. PMLR, 2016.

\bibitem{zimmermann2021computing}
Ralf Zimmermann and Knut H{\"u}per.
\newblock Computing the riemannian logarithm on the stiefel manifold: metrics,
  methods and performance.
\newblock {\em arXiv preprint arXiv:2103.12046}, 2021.

\end{thebibliography}
\bibliographystyle{plain}
}
\newpage
\appendix

\section{Detailed Preliminary Results of Optimization on Riemannian Manifolds}\label{appendix_manifold}
Suppose $\M$ is an $m$-dimensional differentiable manifold. The tangent space $T_x\M$ at $x\in\M$ is a linear subspace that consists of the derivatives of all differentiable curves on $\M$ passing through $x$: $T_x\M:=\{\gamma^{\prime}(0): \gamma(0)=x, \gamma([-\delta, \delta]) \subset \mathcal{M}\text { for some } \delta>0, \gamma \text { is differentiable}\}$. Notice that for every vector $\gamma^{\prime}(0)\in T_x\M$, it can be defined in a coordinate-free sense via the operation over smooth functions: $\forall f\in C^{\infty}(\M)$, $\gamma^{\prime}(0)(f):=\frac{d f\circ \gamma(t)}{dt}\mid_{t=0}$. The Riemannian manifold is a smooth manifold that is equipped with an \textbf{inner product} (called Riemannian metric) on the tangent space, $g(\cdot, \cdot)=\langle \cdot, \cdot \rangle _x : T_x\M \times T_x\M \rightarrow \RR$, that varies smoothly on $\M$.

We first review the notion of the differential between manifolds and the Riemannian gradients here.
\begin{definition}[Differential and Riemannian gradients]
    Let $F:\M\rightarrow \mathcal{N}$ be a $C^{\infty}$ map between two differential manifolds. At each point $x\in\M$, the differential of $F$ is a mapping: $F_*:T_x\M\rightarrow T_x\mathcal{N}$ such that $\forall\xi\in T_x\M$, $F_*(\xi)\in T_x\mathcal{N}$ is given by $(F_*(\xi))(f):=\xi(f\circ F)\in\RR,\ f\in C_{F(x)}^{\infty}(\M)$. 
    
    If $\mathcal{N}=\RR$, i.e. $f\in C^\infty(\M)$, the differential $f_*$ is also denoted as $d f$. For a Riemannian manifold with Riemannian metric $g$, the Riemannian gradient for $f\in C^\infty(\M)$ is the unique tangent vector $\grad f(x)\in T_x\M$ such that $df(\xi) = g(\grad f, \xi),\ \forall \xi\in T_x\M$.
\end{definition}

For the convergence analysis, we also need the notion of exponential mapping and parallel transport. To this end, we need to first recall the definition of a geodesic.
\begin{definition}[Geodesic and exponential mapping]
    Given $x\in\M$ and $\xi\in T_x\M$, the geodesic is the curve $\gamma:I\rightarrow\M$, $0\in I\subset\RR$ is an open set, so that $\gamma(0)=x$, $\Dot{\gamma}(0)=\xi$ and $\nabla_{\Dot{\gamma}}\Dot{\gamma}=0$ where $\nabla:T_x\M\times T_x\M\rightarrow T_x\M$ is the Levi-Civita connection defined by metric $g$. In local coordinates, $\gamma$ is the unique solution of the following second-order differential equations:
    $$
        \frac{d^2\gamma^k}{d t^2} + \Gamma_{i,j}^{k}\frac{d\gamma^i}{d t}\frac{d\gamma^j}{d t}=0
    $$
    under Einstein summation convention, where $\Gamma_{i,j}^{k}$ are Christoffel symbols defined by metric tensor $g$. The exponential mapping $\Exp_x$ is defined as a mapping from $T_x\M$ to $\M$ s.t. $\Exp_x(\xi):= \gamma(1)$ with $\gamma$ being the geodesic with $\gamma(0)=x$, $\Dot{\gamma}(0)=\xi$. A natural corollary is $\Exp_x(t\xi):= \gamma(t)$ for $t\in[0, 1]$. Another useful fact is $d(x,\Exp_x(\xi))=\|\xi\|_x$ since $\gamma'(0)=\xi$ which preserves the speed.
\end{definition}

\section{Proofs}\label{appendix_proof}
In this section we provide the proofs of lemmas and theorems mentioned in the main paper. We first finish the proof of Lemma \ref{lemma_regularization_tangent_mean}:

\begin{proof}[Proof of Lemma \ref{lemma_regularization_tangent_mean}]
    By Cauchy-Schwarz inequality we have
    \[
    \begin{split}
        d(x_{t+1}, x_t) &= \|\Exp_{x_t}^{-1}(x_{t+1})\| \\
        &=\|\frac{1}{k}\sum_{i\in S_t}\Exp_{x_{t}}^{-1}(x^{(i)})\|\leq \frac{1}{k}\sum_{i\in S_t}\|\Exp_{x_{t}}^{-1}(x^{(i)})\| = \frac{1}{k}\sum_{i\in S_t} d(x_t, x^{(i)}).
    \end{split}
    \]
\end{proof}

Now we turn to the proof of Theorem \ref{thm_nonconvex1}. We would utilize the following lemma:

\begin{lemma}\label{lemma_temp1}
    Under the same settings as Theorem \ref{thm_nonconvex1}, we have
    \[
        f(x_{t+1}) - f(x_t)\leq  -\eta_t^{(i)}\|\grad f(x_t)\|^2 + \frac{(\eta_t^{(i)})^2L}{2}\|\grad f(x_t)\|^2.
    \]
\end{lemma}
\begin{proof}[Proof of Lemma \ref{lemma_temp1}]
    From the update we know that
    $$
    x_{\ell+1}^{(i)}\leftarrow \Exp_{x_{\ell}^{(i)}}\left[-\eta_t^{(i)} \left(\grad f_i(x_{\ell}^{(i)}) - P_{x_t\rightarrow x_{\ell}^{(i)}}(\grad f_i(x_t) - \grad f(x_t))\right)\right]
    $$
    i.e.
    $$
    \Exp_{x_{\ell}^{(i)}}^{-1}(x_{\ell+1}^{(i)})\leftarrow -\eta_t^{(i)} \left(\grad f_i(x_{\ell}^{(i)}) - P_{x_t\rightarrow x_{\ell}^{(i)}}(\grad f_i(x_t) - \grad f(x_t))\right).
    $$
    
    When $\tau_i=1$, $x_{0}^{(i)}=x_t$ thus
    $$
    \Exp_{x_t}^{-1}(x_{1}^{(i)})\leftarrow -\eta_t^{(i)} \left(\grad f_i(x_t) - P_{x_t\rightarrow x_{1}^{(i)}}(\grad f_i(x_t) - \grad f(x_t))\right) = -\eta_t^{(i)}\grad f(x_t)
    $$
    
    Using Lipschitz smooth of $f_i$ again and the tangent space mean \eqref{tangent_space_mean}, we have
    \[
    \begin{split}
        f(x_{t+1}) - f(x_t)\leq & \langle \Exp_{x_t}^{-1}(x_{t+1}),\grad f(x_t) \rangle + \frac{L}{2}d^2(x_{t+1}, x_{t}) \\
        = & \langle \frac{1}{k}\sum_{i\in S_t}\Exp_{x_{t}}^{-1}(x_{1}^{(i)}), \grad f(x_t) \rangle + \frac{L}{2}\|\frac{1}{k}\sum_{i\in S_t}\Exp_{x_{t}}^{-1}(x_{1}^{(i)})\|^2 \\
        = & -\eta_t^{(i)}\|\grad f(x_t)\|^2 + \frac{(\eta_t^{(i)})^2L}{2}\|\grad f(x_t)\|^2,
    \end{split}
    \]
    where we used the tangent space mean \eqref{tangent_space_mean} for the first equality.
\end{proof}

Now we are ready to present the proof of Theorem \ref{thm_nonconvex1}.

\begin{proof}[Proof of Theorem \ref{thm_nonconvex1}]
    By taking $\eta^{(i)}\leq \frac{1}{L}$, from Lemma \ref{lemma_temp1} we have
    $$
        f(x_{t+1}) - f(x_t)\leq - \frac{1}{2 L}\|\grad f(x_t)\|^2.
    $$
    Summing this inequality over $t=0,1,\ldots,T$, we obtain
    \[
    \frac{1}{2L}\sum_{t=0}^T \|\grad f(x_t)\|^2 \leq f(x_0) - f(x_{T+1}) \leq f(x_0) - f(x^*),
    \]
    which yields \eqref{thm-ineq} immediately. 
\end{proof}

Before we present the proof of Theorem \ref{thm_nonconvex1.1}, we need the following lemma, which is adopted from~\cite{zhang2016fast}.
\begin{lemma}[Lemma 2 in~\cite{zhang2016fast}]\label{lemma_temp2}
    Consider Algorithm \ref{manifold_fedsvrg} with \textbf{Option} 2. Suppose we run randomly chosen local agent $i$ at the $t$-th outer iteration. If we run the local agent $i$ for $\tau_i$ local gradient steps (\ref{local_update_fedsvrg}) with initial point $x_t$, then it holds:
    \begin{equation}\label{lemma_temp2-ineq}
        \E\|\grad f(x_{\ell}^{(i)})\|^2\leq \frac{R_{\ell} - R_{\ell+1}}{\delta_{\ell}},\ \ell=0,...,\tau_i-1,
    \end{equation}
    where the expectation is taken with respect to the randomly selected index $i$, $R_{\ell}:=\E[f(x_{\ell}^{(i)}) + c_\ell \|\Exp_{x_t}^{-1}(x_{\ell}^{(i)})\|^2]$, $c_\ell=c_{\ell+1}(1+\beta\eta +2\zeta L^2\eta^2) + L^3\eta^2$ and $\delta_{\ell}=\eta  - \frac{c_{\ell+1}\eta }{\beta}-L\eta^2 - 2c_{\ell+1}\zeta\eta^2$. Here $\beta$ is a free constant to be determined and we take $c_{\tau_i}=0$ in the recursive definition.
\end{lemma}

Now we turn to the proof of Theorem \ref{thm_nonconvex1.1}:

\begin{proof}[Proof of Theorem \ref{thm_nonconvex1.1}]
    Since $k=1$, without loss of generality, we denote $i$ as the agent that we choose at the $t$-th iteration. Moreover, we denote $\eta=\eta^{(i)}$ because there is only one agent. 
    
    From \eqref{lemma_temp2-ineq}, we note that if we set $\eta < \frac{1}{L+2 c_{\ell+1}\zeta}(1-\frac{c_{\ell+1}}{\beta})$, then we have $\delta^{(i)}:=\min_{\ell=0,...,\tau_i} \delta_{\ell}>0$. In this case, summing \eqref{lemma_temp2-ineq} over $\ell=0,1,...,\tau_i-1$ yields
    \begin{equation}\label{temp5}
        \frac{1}{\tau_i}\sum_{\ell=0,...,\tau_i-1}\E\|\grad f(x_{\ell}^{(i)})\|^2\leq \frac{R_{0} - R_{\tau_i}}{\tau_i\delta^{(i)}}\leq \E\left(\frac{f(x_t) - f(x_{\tau_i}^{(i)})}{\tau_i\delta^{(i)}}\right),
    \end{equation}
    since $R_{0}=f(x_t)$ and $R_{\tau_i}=\E[f(x_{\tau_i}^{(i)}) + c_\ell \|\Exp_{x_t}^{-1}(x_{\tau_i}^{(i)})\|^2]\geq \E[f(x_{\tau_i}^{(i)})]$. 
    Now we take $\beta=L\zeta^{1/2}/n^{1/3}$ and $\eta = 1/(10 L n^{2/3}\zeta^{1/2})$\footnote{It is straightforward to verify that $\eta<\frac{1}{L+2 c_{\ell+1}\zeta}(1-\frac{c_{\ell+1}}{\beta})$ with this choice of $\eta$ for $\ell=0,...,\tau_i$.}. From the recurrence $c_\ell=c_{\ell+1}(1+\beta\eta +2\zeta L^2\eta^2) + L^3\eta^2$ and $c_{\tau_i}=0$ we have
    $$
    c_0=\frac{L}{100 n^{4/3}\zeta}\frac{(1+\theta)^{\tau_i}-1}{\theta},
    $$
    where
    $$
    \theta=\eta \beta+2 \zeta \eta^{2} L^{2}=\frac{1}{10 n}+\frac{1}{50 n^{4/3}} \in\left(\frac{1}{10 n}, \frac{3}{10 n}\right)
    $$
    is a parameter. If we take $\tau_i=\left\lfloor 10 n/ 3\right\rfloor$ such that $(1+\theta)^{\tau_i}<(1+\frac{3}{10n})^{\tau_i} < e$, then
    $$
        c_0\leq \frac{L}{10 n^{1/3} \zeta}(e-1),
    $$
    and $\delta^{(i)}$ is bounded by
    $$
    \begin{aligned}
    \delta^{(i)}&\geq\left(\eta-\frac{c_{0} \eta}{\beta}-\eta^{2} L-2 c_{0} \zeta \eta^{2}\right) \\
    &\geq \eta\left(1-\frac{e-1}{10\zeta^{3/2}}-\frac{1}{10n^{2/3} \zeta^{1/2}}-\frac{e-1}{50 n \zeta^{1/2}}\right) \\
    &\geq \frac{\eta}{2} = \frac{1}{20 L n^{2/3} \zeta^{1/2}},
    \end{aligned}
    $$
    where the last inequality is by $\zeta, n\geq 1$. Note that this lower bound of $\delta^{(i)}$ is independent from the choice of local agent $i$.
    
    Now summing \eqref{temp5} over $t=0,...,T-1$ with $\delta^{(i)}\geq \frac{\eta}{2}$ we get
    \begin{equation}
        \frac{1}{T}\sum_{t=0,...,T-1}\frac{1}{\tau_i}\sum_{\ell=0,...,\tau_i-1}\E\|\grad f(x_{\ell}^{(i)})\|^2\leq \frac{2\Delta}{\tau \eta T},
    \end{equation}
    where $\Delta= f(x_0) - f^*$.
    
    Now using the {\bf Option 2} of the output of Algorithm \ref{manifold_fedsvrg}, we get
    \[
        \E\|\grad f(\Tilde{x})\|^2\leq \frac{\Delta \rho}{\tau T},
    \]
    where $\rho = \frac{\eta}{2}=\frac{1}{20 L n^{2/3} \zeta^{1/2}}$.
\end{proof}

Before we present the proof of Theorem \ref{thm_geodesic_convex}, we need the following lemma~\cite{zhang2016first}.
\begin{lemma}[Corollary 8 in~\cite{zhang2016first}]\label{lemma_temp3}
    Suppose the sectional curvature of $\M$ is lower bounded by $\kappa_{\min}$ and we update $x_{t+1}\leftarrow \Exp_{x_t}(-\eta_t g_t)$. Suppose also that the update sequence $\{x_t\}\subset\mathcal{D}$ where $\mathcal{D}$ is a compact set with diameter $D$, then for any $x\in\M$ it holds:
    \begin{equation}\label{temp1}
        \langle-g_t, \Exp_{x_{t}}^{-1}(x)\rangle\leq \frac{1}{2\eta_t}(d^2(x_t,x) - d^2(x_{t+1}, x)) + \frac{\zeta\eta_t}{2}\|g_t\|^2.
    \end{equation}
    where $\zeta$ is given in \eqref{zeta_eq}.
\end{lemma}

We now present the proof of Theorem \ref{thm_geodesic_convex}.

\begin{proof}[Proof of Theorem \ref{thm_geodesic_convex}]
    From Lemma \ref{lemma_temp3} we get
    \begin{equation}\label{temp11}
        \langle\frac{1}{k}\sum_{i\in S_t}\Exp_{x_{t}}^{-1}(x^{(i)}), \Exp_{x_{t}}^{-1}(x)\rangle\leq \frac{1}{2}(d^2(x_t,x) - d^2(x_{t+1}, x)) + \frac{\zeta}{2}\|\frac{1}{k}\sum_{i\in S_t}\Exp_{x_{t}}^{-1}(x^{(i)})\|^2,
    \end{equation} 
    which is equivalent to (since we assume $S_t=[n]$ and $\eta^{(i)}=\eta$):
    \begin{equation}\label{temp2}
        -\eta\langle\frac{1}{n}\sum_{i=1,...,n}\grad f_i(x_t), \Exp_{x_{t}}^{-1}(x)\rangle\leq \frac{1}{2}(d^2(x_t,x) - d^2(x_{t+1}, x)) + \frac{\zeta}{2}\|\frac{1}{n}\sum_{i=1,...,n}\Exp_{x_{t}}^{-1}(x^{(i)})\|^2.
    \end{equation}
    
    Now use the geodesic convexity of $f_i$ and \eqref{temp2}, we have (denote $\Delta_t:=f(x_t)-f(x^*)$ and $\Delta^{i}_t:=f_i(x_t)-f_i(x^*)$)
    $$
        \Delta^{i}_t\leq -\langle \grad f_i(x_t), \Exp_{x_{t}}^{-1}(x^*) \rangle.
    $$
    Summing this inequality over $i=1,...,n$, we get
    \begin{equation}\label{temp3}
    \begin{aligned}
        \Delta_t\leq & -\langle \frac{1}{n}\sum_{i=1,...,n}\grad f_i(x_t), \Exp_{x_{t}}^{-1}(x^*) \rangle\\
        \leq & \frac{1}{2\eta}(d^2(x_t,x^*) - d^2(x_{t+1}, x^*)) + \frac{\zeta}{2\eta}\|\frac{1}{n}\sum_{i=1,...,n}\Exp_{x_{t}}^{-1}(x^{(i)})\|^2 \\
        \leq & \frac{1}{2\eta}(d^2(x_t,x^*) - d^2(x_{t+1}, x^*)) + \frac{\zeta\eta}{2 n}\|\grad f(x_t)\|^2.
    \end{aligned}
    \end{equation}
    
    Again from Lemma \ref{lemma_temp1} we get
    \begin{equation}\label{temp4}
        \Delta_{t+1} - \Delta_{t} \leq (-\eta_t^{(i)} + \frac{(\eta_t^{(i)})^2L}{2})\|\grad f(x_t)\|^2.
    \end{equation}
    
    Now multiply \eqref{temp4} by $\zeta$ and add it to \eqref{temp3}, we get
    \begin{equation}
        \zeta\Delta_{t+1} - (\zeta - 1)\Delta_{t}\leq \zeta\left(\frac{\eta}{2 n} -\eta + \frac{\eta^2L}{2}\right)\|\grad f(x_t)\|^2 + \frac{1}{2\eta}(d^2(x_t,x^*) - d^2(x_{t+1}, x^*)).
    \end{equation}
    Now take $\eta\leq \frac{1}{2 L}$, we know that $\frac{\eta}{2 n} -\eta + \frac{\eta^2L}{2}\leq 0$, thus
    \begin{equation}
        \zeta\Delta_{t+1} - (\zeta - 1)\Delta_{t}\leq \frac{1}{2\eta}(d^2(x_t,x^*) - d^2(x_{t+1}, x^*)).
    \end{equation}
    Summing this up over $t$ from $0$ to $T-1$ we get
    \begin{equation}
        \zeta\Delta_{T} + \sum_{t=0}^{T-1}\Delta_{t}\leq (\zeta - 1)\Delta_1 + \frac{d^2(x_0,x^*)}{2\eta}.
    \end{equation}
    Also by \eqref{temp4} we know $\Delta_{t+1} \leq \Delta_{t}$, thus
    \begin{equation}
        \Delta_T\leq \frac{\zeta D^2}{2\eta(\zeta+ T-2)}.
    \end{equation}
\end{proof}

\section{\texttt{RFedAvg} and \texttt{RFedProx} algorithms}\label{appendix_fedprox}
\texttt{FedAvg}~\cite{mcmahan2017communication} and \texttt{FedProx}~\cite{li2020federated} are two widely used algorithms for FL problems in Euclidean space. At each iteration, \texttt{FedAvg} minimizes the local loss $f_i$ for fixed steps using gradient descents:
\begin{equation}\label{fedavg-local-step}
    x_{\ell+1}^{(i)}\leftarrow x_{\ell}^{(i)} - \eta^{(i)}\nabla f_i(x_{\ell+1}^{(i)}),
\end{equation}
while \texttt{FedProx} solves a local proximal point subproblem:
\begin{equation}\label{fedprox-local-step}
    x^{(i)}\leftarrow \argmin_{x} f_i(x) + \frac{\mu}{2}\|x - x_t\|^2.
\end{equation}

For \texttt{RFedAvg}, which is the Riemannian counterpart of \texttt{FedAvg}, \eqref{fedavg-local-step} is replaced by
\[
x_{\ell+1}^{(i)}\leftarrow \Exp_{x_{\ell}^{(i)}}\left(-\eta^{(i)} \grad f_i(x_{\ell}^{(i)})\right).
\]
For \texttt{RFedProx}, which is the Riemannian counterpart of \texttt{FedProx}, \eqref{fedprox-local-step} is replaced by
\begin{equation}\label{rfedprox-local-step}
x_{t+1}^{(i)}\leftarrow \argmin_{x\in\M} f_i(x) + \frac{\mu}{2}d^2(x, x_t),
\end{equation}
where $d(x,y)$ is the geodesic distance between $x$ and $y$. In the implementation of \texttt{RFedProx}, \eqref{rfedprox-local-step} is solved by Riemannian gradient descent:
\begin{equation}\label{temp6}
    x_{\ell+1}^{(i)}\leftarrow \Exp_{x_{\ell}^{(i)}}(-\eta^{(i)}\grad h_i(x_{\ell}^{(i)})), \ \ell=0,...,\tau_i-1.
\end{equation}
\texttt{RFedAvg} and \texttt{RFedProx} are described in Algorithms \ref{manifold_fedavg} and \ref{manifold_fedprox}, respectively.


\begin{algorithm}[!ht]
\SetKwInOut{Input}{input}
\SetKwInOut{Output}{output}
\SetAlgoLined
\Input{$n$, $k$, $T$, $\{\eta^{(i)}\}$, $\{\tau_i\}$}
\Output{$x_T$}
    \For{$t=0,...,T-1$}{
        Uniformly sample $S_t\subset [n]$ with $|S_t|=k$\;
        \For{each agent $i$ in $S_t$}{
            Receive $x_t$ from the central server\;
            \For{$\ell=0,...,\tau_i-1$}{
                $x_{\ell+1}^{(i)}\leftarrow \Exp_{x_{\ell}^{(i)}}\left(-\eta^{(i)} \grad f_i(x_{\ell}^{(i)})\right)$\;
            }
            Send the obtained $x_{\tau_i}^{(i)}$ to the central server\;
        }
        The central server aggregates the points by the tangent space mean \eqref{tangent_space_mean}\;
    }
 \caption{Riemannian FedAvg algorithm}\label{manifold_fedavg}
\end{algorithm}


\begin{algorithm}[!ht]
\SetKwInOut{Input}{input}
\SetKwInOut{Output}{output}
\SetAlgoLined
\Input{$n$, $k$, $T$, $\mu$, $\gamma$}
\Output{$x_T$}
    \For{$t=0,...,T-1$}{
        Uniformly sample $S_t\subset [n]$ with $|S_t|=k$\;
        \For{each agent $i$ in $S_t$}{
            Receive $x_t$ from the central server\;
            Obtain $x^{(i)}\leftarrow \argmin_{x\in\M} f_i(x) + \frac{\mu}{2}d^2(x, x_t)$ upto a $\gamma$ approximate solution\;
            Send the obtained $x^{(i)}$ to the central server\;
        }
        The central server aggregates the points by the tangent space mean \eqref{tangent_space_mean}\;
    }
 \caption{Riemannian FedProx Algorithm} \label{manifold_fedprox}
\end{algorithm}


\end{document}